%% file: main.tex
\documentclass{article} % For LaTeX2e
\usepackage{iclr2027_conference,times}

\input{math_commands.tex}

\usepackage{hyperref}
\usepackage{url}
\usepackage{graphicx}
\usepackage{algorithm}
\usepackage{algpseudocode}
\usepackage{enumitem}
\usepackage{subcaption}
\usepackage{wrapfig}

\title{Scalable Robust MCTS with Anytime Deterministic Optimality Guarantees}

\author{%
\parbox{\textwidth}{%
\centering
\normalfont
\textbf{Tomer Massas}$^{1}$ \quad
\textbf{Vadim Indelman}$^{2,3}$\\[4pt]
{\small
$^{1}$Faculty of Electrical and Computer Engineering, Technion -- Israel Institute of Technology\\
$^{2}$Stephen B. Klein Faculty of Aerospace Engineering, Technion -- Israel Institute of Technology\\
$^{3}$Faculty of Data and Decision Sciences, Technion -- Israel Institute of Technology\\[3pt]
\texttt{tomer.massas@campus.technion.ac.il, vadim.indelman@technion.ac.il}
}
}}

\iclrfinalcopy % Uncomment for camera-ready version, but NOT for submission.

\usepackage{amsthm}
\usepackage{amssymb}
\input{macros}

\begin{document}

\maketitle
\lhead{Preprint}

\begin{abstract}
MDP planners often rely on nominal transition models that may differ from the true dynamics. 
Robust MDP planning accounts for this uncertainty through a set of admissible models, 
but doing so within the limited computation available online is challenging: 
the search tree grows exponentially with the horizon and each robust backup requires an additional optimization.
Existing online robust methods remain limited in scale and offer at most probabilistic guarantees.
We introduce, to the best of our knowledge, the first deterministic anytime guarantees for robust online MDP planning. 
We obtain these guarantees by restricting computation to the sampled search tree 
while deterministically bounding the contribution of unsampled trajectories.
For any fixed policy, we derive lower and upper bounds on the gap between its projected and true robust values under both rectangular and non-rectangular uncertainty sets. 
Under rectangular uncertainty, we propagate these bounds through the robust Bellman operator to bound the optimal robust value.
These certificates are computable from the current sampled tree, hold at every node, tighten as the sampled support expands, and converge to the robust values as full support is recovered.
We develop CAR-MCTS, a bound-guided MCTS planner that uses these deterministic certificates to guide exploration, prune provably suboptimal branches, and terminate once the optimal root action is certified.
The fixed-policy bounds also yield an anytime policy-evaluation mechanism that certifies the robust value of any given policy, independent of how it was obtained.
Experiments verify the deterministic certificates, their tightening under support expansion, substantial branch pruning, root-action certification, and convergence to the optimal robust policy.
To our knowledge, this is the first scalable robust MCTS framework with anytime deterministic guarantees for both online planning and policy evaluation, without enumerating the full state space.
\end{abstract}

\input{sections/introduction}
\input{sections/related_work}
\input{sections/preliminaries}
\input{sections/Simplified_Robust_MDP}
\input{sections/anytime_deterministic_guarantees}

\input{sections/algorithm}

\input{sections/results}
\input{sections/Acknowledgments}

\bibliography{iclr2027_conference}
\bibliographystyle{iclr2027_conference}

\clearpage
\appendix
\input{sections/appendix}

\end{document}

%% file: math_commands.tex
\usepackage{amsmath,amsfonts,bm}

\def\1{\bm{1}}

\DeclareMathAlphabet{\mathsfit}{\encodingdefault}{\sfdefault}{m}{sl}
\SetMathAlphabet{\mathsfit}{bold}{\encodingdefault}{\sfdefault}{bx}{n}

\newcommand{\E}{\mathbb{E}}

%% file: macros.tex
\usepackage{xparse}

\newif\ifTZsuperscript
\TZsuperscriptfalse
\TZsuperscripttrue   % uncomment this if you want P^T_t instead of P_{T,t}

\NewDocumentCommand{\Model}{m o}{%
\ifTZsuperscript
P^{#1}\IfValueT{#2}{_{#2}}%
\else
P_{#1\IfValueT{#2}{,#2}}%
\fi
}

\NewDocumentCommand{\ProjModel}{m o}{%
\ifTZsuperscript
\widehat{P}^{#1}\IfValueT{#2}{_{#2}}%
\else
\widehat{P}_{#1\IfValueT{#2}{,#2}}%
\fi
}

\NewDocumentCommand{\UncSet}{m o}{%
\ifTZsuperscript
\mathcal{P}^{#1}\IfValueT{#2}{_{#2}}%
\else
\mathcal{P}_{#1\IfValueT{#2}{,#2}}%
\fi
}

\NewDocumentCommand{\ProjUncSet}{m o}{%
\ifTZsuperscript
\widehat{\mathcal{P}}^{#1}\IfValueT{#2}{_{#2}}%
\else
\widehat{\mathcal{P}}_{#1\IfValueT{#2}{,#2}}%
\fi
}

\NewDocumentCommand{\RestrictMap}{m o}{%
\ifTZsuperscript
\mathcal{R}^{#1}\IfValueT{#2}{_{#2}}%
\else
\mathcal{R}_{#1\IfValueT{#2}{,#2}}%
\fi
}

\NewDocumentCommand{\PT}{o}{\IfValueTF{#1}{\Model{T}[#1]}{\Model{T}}}
\NewDocumentCommand{\PZ}{o}{\IfValueTF{#1}{\Model{Z}[#1]}{\Model{Z}}}

\NewDocumentCommand{\PTh}{o}{\IfValueTF{#1}{\ProjModel{T}[#1]}{\ProjModel{T}}}
\NewDocumentCommand{\PZh}{o}{\IfValueTF{#1}{\ProjModel{Z}[#1]}{\ProjModel{Z}}}

\NewDocumentCommand{\PThseq}{O{t} O{H-1}}{\ifTZsuperscript\widehat{P}^{T}_{#1:#2}\else\widehat{P}_{T,#1:#2}\fi}
\NewDocumentCommand{\PZhseq}{O{t+1} O{H}}{\ifTZsuperscript\widehat{P}^{Z}_{#1:#2}\else\widehat{P}_{Z,#1:#2}\fi}
\NewDocumentCommand{\PTseq}{O{t} O{H-1}}{\ifTZsuperscript P^{T}_{#1:#2}\else P_{T,#1:#2}\fi}
\NewDocumentCommand{\PZseq}{O{t+1} O{H}}{\ifTZsuperscript P^{Z}_{#1:#2}\else P_{Z,#1:#2}\fi}

\NewDocumentCommand{\UT}{o}{\IfValueTF{#1}{\UncSet{T}[#1]}{\UncSet{T}}}
\NewDocumentCommand{\UZ}{o}{\IfValueTF{#1}{\UncSet{Z}[#1]}{\UncSet{Z}}}

\NewDocumentCommand{\UTh}{o}{\IfValueTF{#1}{\ProjUncSet{T}[#1]}{\ProjUncSet{T}}}
\NewDocumentCommand{\UZh}{o}{\IfValueTF{#1}{\ProjUncSet{Z}[#1]}{\ProjUncSet{Z}}}

\NewDocumentCommand{\UThseq}{O{t} O{H-1}}{\UTh[#1:#2]}
\NewDocumentCommand{\UZhseq}{O{t+1} O{H}}{\UZh[#1:#2]}

\NewDocumentCommand{\UTseq}{O{t} O{H-1}}{\UT[#1:#2]}
\NewDocumentCommand{\UZseq}{O{t+1} O{H}}{\UZ[#1:#2]}

\NewDocumentCommand{\RT}{o}{\IfValueTF{#1}{\RestrictMap{T}[#1]}{\RestrictMap{T}}}
\NewDocumentCommand{\RZ}{o}{\IfValueTF{#1}{\RestrictMap{Z}[#1]}{\RestrictMap{Z}}}

\NewDocumentCommand{\RTseq}{O{t} O{j-1}}{\RT[#1:#2]}
\NewDocumentCommand{\RZseq}{O{t+1} O{j}}{\RZ[#1:#2]}

\newcommand{\UB}{\mathcal{U}_t^{\pi}(h_t)}
\newcommand{\LB}{\mathcal{L}_t^{\pi}(h_t)}
\newcommand{\UBrec}{\mathcal{U}_{t,rect}^{\pi^U}(h_t)}
\newcommand{\LBrec}{\mathcal{L}_{t,rect}^{\pi^L}(h_t)}

\newcommand{\RUDBa}{\mathcal{URDB}_{t}^{\pi^U}(h_t,a_t)}
\newcommand{\RLDBa}{\mathcal{LRDB}_{t}^{\pi^L}(h_t,a_t)}
\newcommand{\RUDB}{\mathcal{URDB}_{t}^{\pi^U}(h_t)}
\newcommand{\RLDB}{\mathcal{LRDB}_{t}^{\pi^L}(h_t)}
\newcommand{\PTht}{\PTh[t]}

\newcommand{\PTt}{\PT[t]}

\NewDocumentCommand{\PThk}{O{k}}{\PTh[#1]}
\NewDocumentCommand{\PZhk}{O{k+1}}{\PZh[#1]}
\NewDocumentCommand{\PTk}{O{k}}{\PT[#1]}
\NewDocumentCommand{\PZk}{O{k+1}}{\PZ[#1]}

\newcommand{\UTht}{\UTh[t]}

\newcommand{\UTt}{\UT[t]}

\newcommand{\RTt}{\RT[t]}

\newcommand{\argmaxover}[1]{\mathop{\mathrm{arg\,max}}\limits_{#1}}
\newcommand{\arginfover}[1]{\mathop{\mathrm{arg\,min}}\limits_{#1}}

\newcommand{\Sin}{\widehat{\mathcal{S}}}

\newcommand{\Afull}{{\mathcal{A}}}
\newcommand{\Sfull}{{\mathcal{S}}}

\NewDocumentCommand{\Tseq}{O{t} O{j}}{\mathcal{T}_{#1:#2}}
\NewDocumentCommand{\Tinseq}{O{t} O{j}}{\widehat{\mathcal{T}}_{#1:#2}}
\NewDocumentCommand{\Toutseq}{O{t} O{j}}{\mathcal{T}^{\mathrm{out}}_{#1:#2}}

\NewDocumentCommand{\tauseq}{O{t} O{j}}{\mathcal{\tau}_{#1:#2}}
\NewDocumentCommand{\tauinseq}{O{t} O{j}}{\widehat{\mathcal{\tau}}^{\mathrm{in}}_{#1:#2}}
\NewDocumentCommand{\tauoutseq}{O{t} O{j}}{\mathcal{\tau}^{\mathrm{out}}_{#1:#2}}

\newtheorem{theorem}{Theorem}
\newtheorem{lemma}{Lemma}
\newtheorem{corollary}{Corollary}

\newcommand{\restatedtheoremname}{}
\newtheorem*{restatedtheorem}{\restatedtheoremname}

\newenvironment{theoremrestated}[1]
{
  \renewcommand{\restatedtheoremname}{Theorem~\ref{#1} (restated)}
  \begin{restatedtheorem}
}
{
  \end{restatedtheorem}
}

\newcommand{\restatedlemmaname}{}
\newtheorem*{restatedlemma}{\restatedlemmaname}

\newenvironment{lemmarestated}[1]
{
  \renewcommand{\restatedlemmaname}{Lemma~\ref{#1} (restated)}
  \begin{restatedlemma}
}
{
  \end{restatedlemma}
}

\newcommand{\restatedcorollaryname}{}
\newtheorem*{restatedcorollary}{\restatedcorollaryname}

\newenvironment{corollaryrestated}[1]
{
  \renewcommand{\restatedcorollaryname}{Corollary~\ref{#1} (restated)}
  \begin{restatedcorollary}
}
{
  \end{restatedcorollary}
}

%% file: sections/introduction.tex
\section{Introduction}
% \TM{[maybe add a teaser figure to illustrate the approach]}\\
Markov decision processes (MDPs) provide a standard framework for sequential
decision making under uncertainty, with the objective of finding a policy that
maximizes the expected cumulative reward. As the state space or planning
horizon grows, solving the MDP globally becomes computationally intractable.
Online planning instead focuses computation on the current decision through a
lookahead from the current state. Common online methods, such as Monte Carlo
Tree Search (MCTS), address this lookahead through sampling, incrementally
constructing and selectively expanding a partial search tree
\citep{kocsis2006bandit}. By allocating simulations to a subset of future
outcomes, MCTS reduces the cost of online planning 
% while refining the action-value estimates at the root
while refining root action-value.
 
Model-based online planning relies on a generative transition model that may
differ from the true dynamics. Consequently, nominal value estimates may be
inaccurate and nominally optimal policies may become suboptimal under model
perturbations \citep{wiesemann2013robust}. Robust Markov decision processes (RMDPs) address this sensitivity by
optimizing against the worst-case model within a set of admissible transition models
\citep{iyengar2005robust,nilim2005robust}, but each robust backup then requires
an additional optimization over transition models. Repeating these backups
throughout the online search tree under a limited decision-time budget further
restricts how much of the tree can be explored. 
 
To improve scalability, robust RL has explored approaches
such as function approximation and policy learning from sampled interactions
under model uncertainty
\citep{tamar2014scaling,wang2021online,panaganti2022sample}.
These methods generally aim to learn a robust policy over the state space,
rather than allocating computation to the specific decision faced online.
More recently, sample-based methods have targeted robust online planning
directly, including Robust Sparse Sampling (RSS)
\citep{shazman2026online} and robust MCTS
\citep{dam2026online}. However, RSS scales exponentially with the planning
horizon, 
% while robust MCTS proposed by  \citet{dam2026online} relies on the full state space,
while \citet{dam2026online} robust MCTS relies on the full state space,
hindering scalability in real world problems that often exhibit large state spaces.

Beyond scalability, online planning acts on a partially explored tree, leaving
the effect of unexplored outcomes on the current decision unknown. This is
particularly important in the robust setting, where such outcomes may alter
the worst-case value and preferred action. For MCTS under a known model,
\citet{shah22or} provide finite-time probabilistic guarantees, which
\citet{dam2026online} 
% leverage in the analysis of their proposed Robust-Power-UCT, establishing finite-sample guarantees for robust MCTS.
leverage to establish finite-sample guarantees for Robust-Power-UCT.
Similarly, RSS provides probabilistic finite-sample guarantees for robust
sample-based planning \citep{shazman2026online}.
In a complementary line of work,
\citet{barenboim2023online,Barenboim26aij} establish deterministic anytime
guarantees for sample-based anytime POMDP planning in the nominal setting. However,
such deterministic anytime guarantees are not available for robust online planning, 
where the effect of unexplored outcomes must be assessed jointly
with model uncertainty.

%\VID{[mention recent work that provides finite-time guarantees for MCTS given a model: finite-time probabilistic guarantees  \citep{shah22or}, that \citep{dam2026online} leverages;  Moran's anytime deterministic guarantees  \citep{barenboim2023online, Barenboim26aij} that we build upon.]}

To address these gaps, we formulate robust online planning directly over the
sampled portion of the search tree while explicitly accounting for the
probability mass associated with unexplored outcomes. Restricting computation
to the sampled tree preserves the scalability required for online planning.
The resulting bounds quantify how much unexplored outcomes may still affect the
current solution. This remaining uncertainty guides further search, focusing
the limited planning budget where it can most effectively tighten the bounds
and ultimately certify the decision.

\textbf{Contributions.}
To the best of our knowledge, this is the first work to provide deterministic anytime bounds for sample-based online planning in robust MDPs. Our main
contributions are:
\begin{itemize}[leftmargin=2em,itemsep=5pt, topsep=0pt, parsep=0pt, partopsep=0pt]

    \item We derive deterministic anytime two-sided bounds on the robust value of a fixed policy under general and non-rectangular uncertainty sets. The bounds tighten monotonically as the sampled support expands and, in the non-robust setting, improve upon prior deterministic sample-based bounds by providing tighter lower and upper guarantees.

    \item Under rectangular uncertainty, we extend the fixed-policy bounds to
    deterministic anytime lower and upper certificates on the \textbf{optimal} robust
    value through robust Bellman recursion.

    \item We introduce \textbf{CAR-MCTS} (\emph{Certified Anytime Robust MCTS}),
    an anytime robust planner with deterministic optimality guarantees computed
    solely from the constructed partial search tree. By avoiding enumeration of the full 
    state space, CAR-MCTS remains applicable to large state spaces while
    retaining online computation. It 
    uses the upper certificate to guide exploration and 
    prunes provably suboptimal actions to focus computation on unresolved branches. 
    Once the bounds certify the optimal root action, the search can terminate 
    immediately and return the certified optimal decision. %\VID{[clarify scalability to large state spaces; clarify \textbf{CAR-MCTS} comes with anytime deterministic optimality guarantees calculated solely based on the constructed partial tree]}

    \item We empirically evaluate CAR-MCTS against existing robust online
    planners, demonstrating certificate tightening, pruning, root-action
    certification, and convergence toward the optimal robust solution.
\end{itemize}

%% file: sections/related_work.tex
\section{Related Work}

Robust online planning has received comparatively limited attention.
\citet{sharma2019robust} introduced RAMCP, a Monte Carlo planner for
risk-sensitive Bayes-adaptive planning, but its formulation relies on a prior
over a parametric family of transition models. \citet{kohankhaki2024monte}
developed UA-MCTS, which incorporates transition uncertainty into MCTS
exploration; its empirical evaluation considers deterministic environments,
while its theoretical analysis is limited to the corrupted-bandit setting. More closely related to our setting, \citet{shazman2026online} introduced
Robust Sparse Sampling (RSS), which extends Sparse Sampling to RMDPs using
Sample Average Approximation and provides finite-sample performance guarantees.
Its complexity is independent of the state-space cardinality but grows
exponentially with the planning horizon. Its analysis is further based on
non-negative rewards and assumes the existence of a zero-reward fail-state
that is absorbing under every model in the uncertainty set. \citet{dam2026online} developed a robust
MCTS planner that combines robust power-mean backups with exploration bonuses
and establishes finite-sample convergence throughout the search tree. While its tree search is selective, the robust transition optimization retains an explicit dependence on the full state space, thereby limiting scalability as the state-space cardinality grows. %\VID{[not scalable for large problems etc.]}.

%% file: sections/preliminaries.tex
\section{Preliminaries}

A robust MDP accounts for model uncertainty by replacing the single transition model with a set of admissible models. We consider a finite-horizon non-stationary RMDP
$\mathcal{M}^{\mathrm{rob}}=(\mathcal S,\mathcal A, \UTseq[0][H-1], r, H)$, where $\mathcal S$ and $\mathcal A$, are the finite state and action spaces. The set $\UTseq[0][H-1]$ contains admissible non-stationary transition model sequences $\PTseq[0][H-1]$, over the horizon $H$. The information available to the agent is summarized by the history of actions and states up to time $t$ which is denoted by $h_t\triangleq(a_{0:t-1}, s_{0:t})$. The reward function $r:\mathcal{S}\times\mathcal A \to [r_{min},r_{max}]$ is assumed to be bounded. A policy specifies how actions are selected based on the history. For a finite-horizon problem, a policy is a sequence $\pi={\{\pi_t}\}_{t=0}^{H-1}$, where each $\pi_t$ maps a history $h_t$ to an action: %$\pi_t(a | h_t) \triangleq \prob(\ActionRV{t}=a | h_t)$ 
$\pi_t : \mathcal H_t \rightarrow \mathcal A$. Given a sequence of nominal models $\mathring{P}_{0:H-1}^{T}$, such that the kernel at time step $k$ is $\mathring{P}_{k}^{T}:\mathcal{S}\times\mathcal{A} \rightarrow\Delta(\mathcal{S})$, Let $\UTseq$ denote a general set of admissible non-stationary transition
model sequences. One common choice is a
ball uncertainty set of the form:
% the non-rectangular uncertainty set over the planning horizon $H$ is defined as:
% 
{
\setlength{\abovedisplayskip}{6pt}
\setlength{\belowdisplayskip}{1pt}
\vspace{-4pt}
\begin{align}\label{eq:uncertainty_set_NR_transition}
    \UTseq
    &=
    \{
    \PTseq: 
    \ D 
    \left(
    \PTseq 
    ||
    \mathring{P}_{t:H-1}^{T}
    \right) 
    \le 
    \rho
    \},
\end{align}
}
\noindent where $D(\cdot,\cdot)$ denotes a distance between sequences of transition
kernels, potentially coupling deviations across time steps and
state-action pairs, and $\rho$ quantifies the uncertainty radius. The robust value function evaluates $\pi$ under the worst-case admissible model, and is defined as:
{
\setlength{\belowdisplayskip}{1pt}
\begin{align} \label{eq:robust value function non-stationary}
    V_t^{\pi,\mathrm{rob}}(h_t)
    =
    \inf_{
    \PTseq \in \UTseq
    }
    V_t^{\pi}(h_t; \PTseq),
\end{align}
}
where $V_t^{\pi}(h_t; \PTseq)$ is the value function given a sequence of models. To reduce clutter, when clear from context, we omit the explicit dependence on $\PTseq$, and write $V_t^{\pi}(h_t; \PTseq)$ as:
\vspace{-2pt}
\begin{align} \label{eq:value function}
    V_{t}^{\pi }( h_{t}) =
    r\bigl(s_t,\pi_t(h_t)\bigr)+
    \mathbb{E}^{\pi, \PT }_{s_{t+1:H-1}}
    \left[ \sum _{j=t+1}^{H-1} r\bigl(s_{j} ,\pi_j(h_j)\bigr) \mid h_t\right] .
\end{align}
The non-rectangular formulation in  \eqref{eq:uncertainty_set_NR_transition} and \eqref{eq:robust value function non-stationary}, forces the adversary to choose the model sequence jointly, so the admissibility of the model at one time may depend on choices at others. This coupling prevents a dynamic-programming decomposition of the Bellman backup, motivating the rectangular relaxation introduced by \citet{iyengar2005robust,nilim2005robust}. The rectangular set at time step $t$ is defined as $\mathcal{P}_t^{T,\mathrm{rect}}=\bigotimes_{\substack {s\in \mathcal{S} \\ a\in \mathcal{A}}} \UTt(s,a)$, and the state-action uncertainty set is:
\begin{align} \label{eq:rect uncertainty set}
    % \UTt^{,rect}
    % &=
    % \bigotimes_{\substack {s\in \mathcal{S} \\ a\in \mathcal{A}}}
    % \UTt(s,a), 
    % \\
    \UTt(s,a)
    &= 
    \Bigl \{ 
    \PTt(\cdot|s,a):  
    D 
    \left(
    \PTt(\cdot|s,a) 
    ||
    \mathring{P}_{t}^{T}(\cdot|s,a)
    \right) 
    \le 
    \rho_{s,a}
    \Bigr \}.
\end{align} 
Rectangularity is imposed across time and transition rows. First, nature may choose a
different worst-case model at each stage independently of its choices at
other stages. Second, the uncertainty set is $(s,a)$-rectangular, such that
each transition row $\PTt(\cdot\mid s,a)$ is selected independently. Since we represent the planning problem as a history tree, we additionally
adopt a history-rectangular interpretation: each history-action node
$(h_t,a)$ is associated with an independent local
uncertainty set $\UTt(s_t,a)$. Thus, nature may select different admissible
transition rows at distinct history nodes. This represents an adversary that may adapt its models choices locally rather than committing in advance to a fixed sequence of models. Under this rectangular formulation, the robust value in \eqref{eq:robust value function non-stationary} admits the recursive form:
% 
% Adapting \eqref{eq:robust value function non-stationary} to the rectangular setting, we denote the rectangular-robust value function, that admits robust DP:
% 
\begin{align} \label{eq:robust-rect-recursive-value}
    V_{t,\mathrm{rect}}^{\pi,\mathrm{rob}}(h_t)
    =
    r\bigl(s_t,\pi_t(h_t)\bigr)
    + \
    \inf_{\PTt \in \UTt}
    \mathbb{E}_{s_{t+1}}^{\PTt}
    \left[
    V_{t+1,\mathrm{rect}}^{\pi,\mathrm{rob}}(h_{t+1})
    \right].
\end{align}
The rectangular relaxation comes at a cost of a more conservative robust value: since $\mathcal{P}_{t:H} \subseteq  \mathcal{P}_{t:H}^{rect}$, enlarging the uncertainty set can only decrease the robust value, yielding $V_{t,\mathrm{rect}}^{\pi,\mathrm{rob}}(h_t)
\le V_{t}^{\pi,\mathrm{rob}}(h_t)$.

\textbf{Monte Carlo Tree Search.}
Monte Carlo Tree Search (MCTS) incrementally constructs a finite lookahead
tree rooted at the current history $h_t$ through repeated simulations
\citep{kocsis2006bandit}. During each simulation, at a decision node $h_k$,
$k \ge t$, an action is selected according to an exploration-strategy that
balances exploration and exploitation, and a successor state is sampled from
a generative transition model. If the resulting successor has not previously
been encountered, the tree is expanded with the new history
$h_{k+1}=(h_k,a_k,s_{k+1})$. The simulation then continues until the planning
horizon is reached, after which the resulting values are backed up along the
visited path. Repeating this process progressively expands the sampled portion
of the planning tree while concentrating computation on branches favored by
the search strategy. At decision time, the planner selects an action at the
root based on the resulting action-value estimates. In our setting, this
sampled tree defines the support of the simplified RMDP introduced in
Section~\ref{sec:Simplified Robust MDP}.

%% file: sections/Simplified_Robust_MDP.tex
\section{Simplified Robust MDP} \label{sec:Simplified Robust MDP}

Computing the exact optimal robust value online is generally intractable:
the planning tree grows exponentially with the horizon, while each robust
backup additionally requires optimization over the transition uncertainty set.
We therefore construct a projected planning problem defined only over the
portion of the tree represented by the current samples. Importantly, the projected transition probabilities are not renormalized; probability mass assigned to unsampled next-state outcomes is retained implicitly and will later be used to bound the approximation error deterministically. At a history-action node $h_{k+1}^-=(h_k,a_k)$, let
$\Sin(h_{k+1}^-)\subseteq\Sfull$ denote the sampled next-state support. For compactness, we use
$s_{t+1:j}\in\Sin(h_j^-)$ to denote that
$s_{k+1}\in\Sin(h_{k+1}^-)$ for every $k=t,\ldots,j-1$.
% At a history node $h_k$, let $\Ain(h_k)\subseteq\Afull$ denote the
% represented action support, and let $\Sin(h_{k+1}^-)\subseteq\Sfull$
% denote the sampled next-state support after selecting $a_k$, where
% $h_{k+1}^-=(h_k,a_k)$. A trajectory $
% \tau_{t:H}=(s_t,a_t,s_{t+1},a_{t+1},\ldots,a_{H-1},s_H)
% $ is represented in the sampled tree if $a_k\in\Ain(h_k)$ and $s_{k+1}\in\Sin(h_{k+1}^-)$ for every $k\in[t,H-1]$.
% We denote the set of all such trajectories by
% $\Tinseq[t][H](h_t)$.
% Thus, a trajectory leaves the sampled tree as soon as either one of its
% states or actions is not represented at the corresponding history node.

% ==================================
% Projected models
% ==================================
\subsection{Projected Transition Model}

For each history-action node $h_{t+1}^-=(h_t,a_t)$, we define the projected transition
model by restricting the full transition model to the sampled next-state
support:
{
\setlength{\abovedisplayskip}{4pt}
\setlength{\belowdisplayskip}{4pt}
\begin{align} \label{eq:projected-transition-model}
    \PTht(s_{t+1}\mid s_t,a_t;h_{t+1}^-)
    \triangleq
    \PTt(s_{t+1}\mid s_t,a_t),
    \qquad
    s_{t+1}\in\Sin(h_{t+1}^-).
\end{align}
}
The projected probabilities are not renormalized. Hence,
$\PTht$ is generally a sub-probability vector, where the missing mass corresponds exactly to transitions leaving the
sampled support. Restricting the transition model in this manner reduces each robust backup from an optimization over $\mathbb{R}^{|\Sfull|}$ to one over
$\mathbb{R}^{|\Sin(h_{t+1}^-)|}$. This reduction makes each robust backup depend only on the sampled support and constitutes the main computational saving of the projected planner.
Using these we define the simplified value function: 
{
\setlength{\abovedisplayskip}{0pt}
\setlength{\belowdisplayskip}{2pt}
\begin{align} \label{eq:simp value function}
    &\widehat{V}_t^{\pi}(h_t)
    =
    \sum_{j=t}^{H-1}
    \sum_{s_{t+1:j} \in \Sin(h_j^-)}
    \left[
    \prod_{k=t}^{j-1}
    \PThk\bigl(s_{k+1}\mid s_k,\pi_k(h_k)\bigr)
    \right]
    r\bigl(s_j,\pi_j(h_j)\bigr).
\end{align}
}
Unlike an ordinary value function, $\widehat V_t^\pi$ is not an expectation under a normalized transition model; it is the contribution of trajectories represented in the sampled tree.
% ==================================
% Projected uncertainty set
% ==================================
\subsection{Projected Uncertainty Set}
Simply imposing the original uncertainty constraint directly in the
lower-dimensional space does not, in general, preserve the set of models
induced by the original RMDP. For a full transition model $P_k^T$, we define
the restriction operator $\RT[k](\PTk;h_{k}) \triangleq \PThk(s_{k+1}| s_k,a_k), \ \ s_{k+1} \in \Sin(h_{k+1}^-)$.
% 
% \begin{align}
%     \RT[k](\PTk;h_{k})
%     \triangleq
%     \PThk(s_{k+1}| s_k,a_k), \quad s_{k+1} \in \Sin(h_{k+1}^-),
% \end{align}
% 
Let $\mathcal R_{t:H-1}^T$ apply this restriction over the planning
horizon. We define the projected uncertainty set as the image of the original
uncertainty set under this restriction:
\begin{align}
\label{eq:projected_uncertainty_set_NR_transition}
\UThseq(h_{H-1})
\triangleq
\left\{
\mathcal R^T_{t:H-1}(P^T_{t:H-1})
:
P^T_{t:H-1}\in\UTseq
\right\}.
\end{align}
Thus, every projected model in $\UThseq$ is the restriction of at least
one admissible full model; equivalently, its omitted components can be
completed to a model belonging to $\UTseq$. We define the resulting simplified-RMDP, as the tuple $\widehat{\mathcal{M}}^{\mathrm{rob}}=(\Sin,\Afull, \UThseq, r, H)$. Given a policy $\pi$, its simplified robust value is the worst-case projected
value over all projected models consistent with the original uncertainty set:
{
\setlength{\abovedisplayskip}{0pt}
\setlength{\belowdisplayskip}{0pt}
\begin{equation}
\label{eq:simplified-robust-value-function-non-stationary}
\widehat V_t^{\pi,\mathrm{rob}}(h_t)
=
\inf_{\PThseq\in\UThseq}
\widehat V_t^\pi(h_t;\PThseq).
\end{equation}
}
Thus, $\widehat V_t^{\pi,\mathrm{rob}}$ applies the same worst-case model
criterion as the original robust value, but evaluates only the trajectories
represented in the sampled tree.

%% file: sections/anytime_deterministic_guarantees.tex
\section{Anytime Deterministic Guarantees}
Having defined the simplified RMDP, we now characterize the approximation induced by restricting planning to the sampled tree. We first derive deterministic anytime certificates for the robust value of an arbitrary fixed policy under general uncertainty sets. We then exploit rectangularity to propagate these certificates recursively through the robust Bellman optimality operator and obtain deterministic bounds on the optimal robust value.
% 
% ==================================
% Fixed-Policy Guarantees
% ==================================
\subsection{Fixed-Policy Guarantees}
We first consider an arbitrary policy \(\pi\) and quantify the error introduced by evaluating it on the simplified RMDP rather than the full RMDP. For a projected model sequence \(\widehat P_{t:H-1}^T\), we define the cumulative in-support mass: $m_t^{\pi}\left(\PThseq \right) \triangleq \sum_{j=t}^{H-1} \sum_{s_{t+1:j} \in \Sin(h_j^-)} \prod_{k=t}^{j-1} \PThk\bigl(s_{k+1}\mid s_k,\pi_k(h_k)\bigr)$. At each future step \(j\), the inner sum is the probability mass that remains within the sampled tree up to \(j\). Hence, \(m_t^\pi\) measures the cumulative coverage of the policy-induced future by the current sampled tree.
% 
% ==================================
% Theorem 1 - foxed policy
% ==================================
\begin{theorem}[Fixed-policy deterministic bounds]\label{thm:bounds}
Let $h_t$ be the history at time $t$, and $V_{t}^{\pi,\mathrm{rob}}(h_t)$ and $\widehat V_t^{\pi,\mathrm{rob}}(h_t)$ denote the robust values of the same policy \(\pi\) in the full \eqref{eq:robust value function non-stationary} and simplified \eqref{eq:simplified-robust-value-function-non-stationary} RMDPs, respectively. For any transition uncertainty set $\UTseq$ and its projected counterpart $\UThseq$:
\begin{align}
    \LB
    &\leq
     V_t^{\pi,\mathrm{rob}}(h_t) - \widehat V_t^{\pi,\mathrm{rob}}(h_t)
    \ \leq \
    \UB,
    \\[0.5em]
    \llap{\text{where}\qquad}
    \UB
    &\triangleq
    \sup\limits_{\PThseq \in \UThseq}
    \left[
    r_{max}
    \biggl (
    H-t-
    m_t^{\pi}\left(\PThseq \right)
    \biggr)\right],
    \\
    \LB
    &\triangleq
    \inf\limits_{\PThseq \in \UThseq}
    \left[
    r_{min}
    \biggl (
    H-t-
    m_t^{\pi}\left(\PThseq \right)
    \biggr)\right].
\end{align}
\end{theorem}
\vspace{-0.5em}
The proof is provided in the Appendix \ref{thm:thm_1_proof}.
The quantity \(H-t-m_t^\pi\) is the cumulative probability mass not represented by the sampled tree over the remaining horizon. Since the reward is bounded as \(r_{\min}\le r(s,a)\le r_{\max}\), Theorem \ref{thm:bounds} bounds the value contribution of the unsampled portion of the tree using the reward range. Importantly, the result does not rely on rectangularity or on a particular structure of the transition uncertainty set, and thus applies to general robust MDP formulations. 
% 
% ==================================
% corollary 1 
% ==================================
\begin{corollary}[Asymptotic convergence] \label{cor:Bound-diminishing}
For any policy $\pi$, as the sampled state support expands to the full state space at every node $\Sin(h_k^-) \uparrow \Sfull, \ k=t,\ldots, H$, the simplified robust value \eqref{eq:simplified-robust-value-function-non-stationary} convergences to the true robust value \eqref{eq:robust value function non-stationary}:
\begin{align}
    \widehat V_t^{\pi,\mathrm{rob}}(h_t)
    \xrightarrow{
    % \substack{
    % \Sin \rightarrow \Sfull}
    }
    V_t^{\pi,\mathrm{rob}}(h_t).
\end{align}
\end{corollary}
The proof is provided in the Appendix \ref{cor:cor_1_proof}.
% 
% ==================================
% corollary 2 
% ==================================
\begin{corollary}[Monotonic tightening under support expansion]
\label{cor:monotonic-tightening}
Let $n$ denote the MCTS simulation index. Since the search tree is expanded
incrementally, the sampled supports are nested across simulations:
$
\Sin^{(n)}(h_k^-)
\subseteq
\Sin^{(n+1)}(h_k^-),
\ k=t,\ldots,H$.
Thus the deterministic bounds of Theorem~\ref{thm:bounds} tighten
monotonically:
{
\setlength{\abovedisplayskip}{4pt}
\setlength{\belowdisplayskip}{0pt}
\begin{align}
\left|\mathcal{L}_{t}^{\pi,(n+1)}(h_t)\right|
\le
\left|\mathcal{L}_{t}^{\pi,(n)}(h_t)\right|,
\qquad
\left|\mathcal{U}_{t}^{\pi,(n+1)}(h_t)\right|
\le
\left|\mathcal{U}_{t}^{\pi,(n)}(h_t)\right|.
\end{align}
}
\end{corollary}
Indeed, expanding the sampled support can only increase the in-support
probability mass, such that
$
m_t^{\pi,(n)}
\le
m_t^{\pi,(n+1)}.
$
Consequently, the cumulative unrepresented mass
$H-t-m_t^\pi$ decreases, and both bounds contract monotonically toward zero. Thus, adding sampled support can only improve the certificate. The convergence
rate is governed by how quickly the planner captures probability mass in the
sampled support. Different support-expansion strategies may therefore yield
different rates of certificate tightening without affecting its validity.

\textbf{Anytime Certified Policy Evaluation.}
The bounds in Theorem~\ref{thm:bounds} depend only on the current sampled
support, the projected uncertainty set, and the evaluated policy, and are
therefore computable from the current partial tree provided the corresponding optimization over the projected uncertainty set can be solved. Consequently, they provide an
anytime deterministic certificate on the robust value of any policy, including
under non-rectangular uncertainty sets, independently of how that policy was obtained. %\VID{[mention also a non-rectangular uncertainty set?]}

\textbf{Recursive Form Under Rectangular Uncertainty.}
Rectangularity enables a recursive robust Bellman decomposition, allowing the
lower and upper fixed-policy certificates to be propagated through the
planning tree. Accordingly, the fixed-policy certificates of
Theorem~\ref{thm:bounds} admit a recursive action-conditioned form. The recursive quantity introduced below is directly related to the cumulative
in-support mass $m_t^\pi$ of Theorem~\ref{thm:bounds}. Under rectangular
uncertainty,
$
1+\delta_t^\pi(s_t\mid h_t)
=
\inf_{\PThseq\in\UThseq}
m_t^\pi(\PThseq).
$
The unit term corresponds to the current stage, while $\delta_t^\pi$ represents the worst-case
cumulative in-support mass from the next step onward. Rectangularity allows
this global minimization to be decomposed recursively as:
{
\setlength{\abovedisplayskip}{3pt}
\setlength{\belowdisplayskip}{2pt}
\begin{align} 
    \delta_{k}^{\pi}(s_k|h_k)
    &=
    \inf_{\PThk \in \UTh[k]}
    \delta_{k}^{\pi, \PThk}(s_k, \pi_k(h_k))
    \qquad k=t,\ldots, H-2,
    \\ 
    \delta_{k}^{\pi,\PThk}(s_k,a_k|h_k)
    &=
    % \inf_{\PThk \in \UTh[k]}
    \sum _{s_{k+1} \in \Sin(h_{k+1}^-)}
    \PThk(s_{k+1}\mid s_{k},a_{k})\,
    \left [
    1 
    +
    \delta_{k+1}^{\pi}(s_{k+1}|h_{k+1})
    \right ],
\end{align}
}
with a terminal condition $\delta_{H-1}^{\pi}(s_{H-1}) \equiv 0$. The term \(1+\delta_{t+1}^\pi\) accounts for the probability mass retained at the next state together with all subsequent in-support mass along the continuation. The action-conditioned bounds are defined as: 
{
\setlength{\abovedisplayskip}{3pt}
\setlength{\belowdisplayskip}{2pt}
\begin{align} 
    &\mathcal{L}_{t,rect}^{\pi,\PTht}(h_t,a_t) \label{eq:action-condition-L}
    \triangleq
    r_{min}
    \Bigl (
    H -1- t
    -
    \delta_{t}^{\pi,\PTht}\bigl(s_t,a_t|h_t)
    \Bigr ),
    \\
    &\mathcal{U}_{t,rect}^{\pi,\PTht}(h_t,a_t) \label{eq:action-condition-U}
    \triangleq
    r_{max}
    \Bigl (
    H -1- t
    -
    \delta_{t}^{\pi,\PTht}\bigl(s_t,a_t|h_t)
    \Bigr ).
\end{align}
}
For clarity, we present the recursive bounds for the common case \(r_{\min}\leq 0\leq r_{\max}\). The remaining reward-sign cases follow analogously by replacing the corresponding infimum with a supremum when the multiplying reward bound is negative; the full expressions are given in the appendix.
% ==================================
% Theorem 2 - Optimality Guarantees 
% ==================================
\subsection{Optimality Guarantees}
Building on the recursive fixed-policy certificates derived above for
rectangular uncertainty, we now extend them to bound the \textbf{optimal robust
value}. By optimizing the lower and upper recursive backups over actions, we
obtain deterministic anytime bounds on
$V_{t,\mathrm{rect}}^{*,\mathrm{rob}}$. Using the action-conditioned bounds \eqref{eq:action-condition-L} and \eqref{eq:action-condition-U}, for an arbitrary continuation policy $\pi$, we define the lower and upper
robust deterministic backups as: %\VID{[above you have $\mathcal{L}_{t,rect}^{\pi}$ but below you switch to $\mathcal{L}_{t,rect}^{\pi^L,\PTht}$; at this point it is also unclear what $\pi^L$ and $\pi^U$ correspond to]}
{
\setlength{\abovedisplayskip}{0pt}
\setlength{\belowdisplayskip}{1pt}
\begin{align}  
    \mathcal{LRDB}_{t}^{\pi}(h_t,a_t) \label{eq:RLDBa}
    &\triangleq
    r(s_t,a_t)
    +
    \inf_{\PTht \in \UTht}
    \Biggl(
    \widehat{\mathbb{E}}_{s_{t+1}|h_t}^{\PTht}
    \!\left[
    \widehat V_{t+1,rect}^{\pi,\widehat{P}_{t+1}^{T, \mathcal{LRDB}}}(h_{t+1}) 
    \right]
    +
    \mathcal{L}_{t,rect}^{\pi,\PTht}(h_{t},a_{t})
    \Biggr),
    \\
    \mathcal{URDB}_{t}^{\pi}(h_t,a_t) \label{eq:RUDBa}
    &\triangleq
    r(s_t,a_t)
    +
    \inf_{\PTht \in \UTht}
    \Biggl(
    \widehat{\mathbb{E}}_{s_{t+1}|h_t}^{\PTht}
    \!\left[
    \widehat V_{t+1,rect}^{\pi,\widehat{P}_{t+1}^{T, \mathcal{URDB}}}(h_{t+1}) 
    \right]
    +
    \mathcal{U}_{t,rect}^{\pi,\PTht}(h_{t},a_{t})
    \Biggr).
\end{align}
}
We use $\widehat P_{t+1}^{T,\mathcal{LRDB}}$ and
$\widehat P_{t+1}^{T,\mathcal{URDB}}$ to denote the projected continuation
models induced by the minimizing model choices of the corresponding recursive
LRDB and URDB backups from time $t+1$ onward. Importantly, the projected value and its correction term are evaluated under the same transition model inside each robust backup. Thus, the adversarial model jointly determines both the projected return and the corresponding approximation correction, rather than minimizing them independently. We now leverage these recursive deterministic backups to establish deterministic anytime optimality guarantees for the rectangular RMDP.
\begin{theorem}[Optimal robust value bounds] \label{thm:optimal-bounds}
Let $h_t$ be an expanded history node at time step $t$, and let $\pi^L$ and $\pi^U$ be the lower- and upper-bound policies defined in \eqref{policices_L_U}. Under rectangular transition
uncertainty, the optimal robust value is bounded deterministically by:
{
\setlength{\abovedisplayskip}{3pt}
\setlength{\belowdisplayskip}{2pt}
\begin{equation}
    \mathcal{LRDB}_{t}^{\pi^L}\left(h_t, \pi_t^L(h_t)\right)
    \leq
    V_{t,\mathrm{rect}}^{*,\mathrm{rob}}(h_t)
    \leq
    \mathcal{URDB}_{t}^{\pi^U}\left(h_t, \pi_t^U(h_t)\right),
\end{equation}
}
where $\pi_t^L$ and $\pi_t^U$ are the lower- and upper-bound policies defined greedily with respect to their corresponding deterministic backups:
{
\setlength{\abovedisplayskip}{3pt}
\setlength{\belowdisplayskip}{0pt}
\begin{align} \label{policices_L_U}
    \pi_t^L(h_t) 
    =
    \argmaxover{a_t \in \mathcal{A}_t}
    \Bigl\{
    \RLDBa
    \Bigr\}, \quad
    \pi_t^U(h_t) 
    =
    \argmaxover{a_t \in \mathcal{A}_t}
    \Bigl\{
    \RUDBa
    \Bigr\}.
\end{align}
}
\end{theorem}
\vspace{-0.5em}
Theorem~\ref{thm:optimal-bounds} assumes that all actions in $\Afull$ are initialized at every expanded decision node. The proof is provided in the Appendix \ref{thm:thm_2_proof}.
\begin{corollary}[Asymptotic optimality]
\label{cor:optimal-convergence}
As the sampled state support expands to the full state space at every node,
$
\Sin(h_k^-)\uparrow\Sfull,
\ k=t,\ldots,H,
$
both deterministic bounds converge to the optimal rectangular robust value:
\vspace{-0.5em}
{
\setlength{\abovedisplayskip}{0pt}
\setlength{\belowdisplayskip}{0pt}
\begin{align}
\mathcal{LRDB}_{t}^{\pi^L}
\left(h_t,\pi_t^L(h_t)\right)
&\longrightarrow
V_{t,\mathrm{rect}}^{*,\mathrm{rob}}(h_t),
\\
\mathcal{URDB}_{t}^{\pi^U}
\left(h_t,\pi_t^U(h_t)\right)
&\longrightarrow
V_{t,\mathrm{rect}}^{*,\mathrm{rob}}(h_t).
\end{align}
}
\end{corollary}
\vspace{-0.3em}
The proof is provided in the Appendix \ref{cor:cor_2_proof}.\\
\textbf{Action pruning and certification.} The lower and upper action-value certificates induce a deterministic ordering test between actions. For two actions \(a,a'\), if: $\ \mathcal{URDB}
\left(h,a\right) < \mathcal{LRDB}\left(h,a'\right)$, then \(a\) is provably suboptimal and its subtree can be safely pruned. Likewise, an action \(a^*\) is certified optimal when: $\ \mathcal{LRDB}
\left(h,a^*\right) > \max_{a \neq a^*} \mathcal{URDB}\left(h,a\right)$. This eliminates redundant search in branches that can no longer affect the optimal decision, allowing computation to be focused on unresolved regions of the tree. At the root, this provides an anytime stopping criterion: planning can terminate as soon as the optimal root action is certified, rather than exhausting the search budget.

\vspace{-0.3em}
\textbf{Robust backup computation.}
For the TV uncertainty sets in our experiments, each inner robust optimization
reduces to a robust expectation over the sampled successors and one aggregate
unsampled state, solvable in
$O(|\Sin(h,a)|\log|\Sin(h,a)|)$ time. See
Appendix~\ref{app:tv-robust-backup}.

\vspace{-0.3em}
Together, Theorems~\ref{thm:bounds} and~\ref{thm:optimal-bounds} provide runtime-accessible deterministic certificates for both policy evaluation and optimal robust planning. We next use these certificates within our MCTS planner and evaluate their effect on convergence, pruning, and root-action certification.

%% file: sections/algorithm.tex
\section{Algorithm}
% 
% We instantiate the proposed certificates in \textbf{CAR-MCTS}
% (\emph{Certified Anytime Robust MCTS}), summarized in
% Algorithm~\ref{alg:car-mcts}. At each simulation, the planner descends the
% current search tree by selecting the action with the largest upper certificate
% $\mathcal{URDB}$ and samples the next state from the nominal transition model.
% Previously unseen successors are added to the sampled support, after which the
% projected robust values and certificates are backed up along the visited path, avoiding explicit enumeration of the full state space and thereby preserving scalability to large state spaces.
% 
\begin{algorithm}[h]
\small
\caption{CAR-MCTS: Certified Anytime Robust MCTS}
\label{alg:car-mcts}
\begin{minipage}[t]{0.60\linewidth}
\vspace{0pt}

\begin{algorithmic}[1]

\Statex \textbf{Input:} root state $s_0$, horizon $H$, budget $N$
\Statex \textbf{Output:} root action $a$, certification flag $c$, action-value bounds

\State $h_0 \gets \Call{NewNode}{s_0,0}$
\State $n \gets 0$

\While{$n < N$}

    \State $(c,a^\star) \gets \Call{Certify}{h_0}$

    \If{$c$}
        \State \Return $(a^\star,\textsc{true},[\mathcal{LRDB}(s_0,a^\star),\mathcal{URDB}(s_0,a^\star)])$
    \EndIf

    \State $h \gets h_0$
    \State $\textit{path} \gets \emptyset$

    \While{$t(h) < H-1$ \textbf{and} $n < N$}

        \State $a \gets
        \arg\max_{a} \mathcal{URDB}_{t(h)}(h,a)$

        \State $s' \sim
        \mathring{P}^{T}_{t(h)}
        (\cdot \mid s(h),a)$

        \State $n \gets n+1$

        \State append $(h,a)$ to $\textit{path}$

        \State $h \gets
        \Call{ExpandIfNeeded}{h,a,s'}$

    \EndWhile

    \ForAll{$(h,a)$ in reverse$(\textit{path})$}

        \State \Call{Backup}{$h,a$}

        \State \Call{Prune}{$h$}

    \EndFor

\EndWhile

\State $a \gets
\arg\max_a \widehat Q^{\mathrm{rob}}_0(s_0,a)$

\State \Return $(a,\textsc{false},[\mathcal{LRDB}(s_0,a),\mathcal{URDB}(s_0,a)])$

\end{algorithmic}

\end{minipage}
\hfill
% \hspace{0.005\linewidth}
\begin{minipage}[t]{0.37\linewidth}
\vspace{0pt}
\small
\raggedright

\setlength{\abovedisplayskip}{3pt}
\setlength{\belowdisplayskip}{3pt}
\setlength{\abovedisplayshortskip}{3pt}
\setlength{\belowdisplayshortskip}{3pt}

\textbf{Operations:}

\medskip

\textbf{\textsc{NewNode}}$(s,t)$

Initialize all actions $a\in\mathcal A$.

\bigskip

\textbf{\textsc{ExpandIfNeeded}}$(h,a,s')$

If $s'$ is not in the sampled support, add it and create
the corresponding child node. Return the child.

\bigskip

\textbf{\textsc{Backup}}$(h,a)$

Update $\widehat Q^{\mathrm{rob}}$, $\mathcal{LRDB}$,
and $\mathcal{URDB}$.

\bigskip

\textbf{\textsc{Prune}}$(h)$

Prune action $a$ if
\[
\mathcal{URDB}(h,a)
<
\max_{a'}\mathcal{LRDB}(h,a').
\]

\medskip

\textbf{\textsc{Certify}}$(h)$

Let
\[ a^\star =
\arg\max_a \mathcal{LRDB}(h,a).
\]
Return true if
\[
\mathcal{LRDB}(h,a^\star)
>
\max_{a\neq a^\star}
\mathcal{URDB}(h,a).
\]

\end{minipage}

\end{algorithm}
We instantiate the proposed certificates in \textbf{CAR-MCTS}
(\emph{Certified Anytime Robust MCTS}), summarized in
Algorithm~\ref{alg:car-mcts}. At each simulation, the planner descends the
current search tree by selecting the action with the largest upper certificate
$\mathcal{URDB}$ and samples the next state from the nominal transition model.
Previously unseen successors are added to the sampled support, after which the
projected robust values and certificates are backed up along the visited path, avoiding explicit enumeration of the full state space and thereby preserving scalability to large state spaces. During the backup, actions whose upper certificate falls below a competing lower certificate are pruned at every visited node. At the root, planning terminates early whenever an action is certified optimal. If certification is not reached within the sampling budget, CAR-MCTS returns the action with the
largest projected robust action-value $\widehat Q_0^{\mathrm{rob}}$ together with deterministic optimality certificates.

%% file: sections/results.tex
\section{Results}

\textbf{Experimental setup.}
We evaluate CAR-MCTS in the Frozen Lake domain with varying grid
sizes, terminal holes, and different start and goal cells. At each nonterminal
cell, the agent chooses one of four actions: \textsc{left}, \textsc{down},
\textsc{right}, or \textsc{up}. Under the nominal transition model, the agent
moves toward the intended neighboring cell with probability $1-p_{\mathrm{slip}}$;
the remaining probability $p_{\mathrm{slip}}$ is distributed over cells in
concentric Chebyshev-distance rings around the current cell, up to a
configurable radius $k$. This ring model allows the number of possible
successor states to be varied across experiments. Transition uncertainty is
modeled by rectangular total-variation sets around the nominal transition
rows. The reward function and remaining experiment-specific parameters, including
the grid size, horizon, $p_{\mathrm{slip}}$, ring radius $k$, and uncertainty
radius, are provided in the Appendix.

% % % ==================================
% % % Figure - value convergance
% % % ==================================
% \begin{figure}[h]
%     \centering
%     \includegraphics[
%     width=0.9\linewidth,
%     height=7cm,
%     keepaspectratio
%     ]{Figures/fl_fig1_V14.png}
%     \caption{Fixed-policy certificate tightening as the sampled support expands.}
%     \label{fig:fixed-policy-convergance}
% \end{figure}

% \begin{wrapfigure}[5]{r}{0.35\linewidth}
%     \centering
%     \includegraphics[width=\linewidth]{Figures/fl_fig1_V14.png}
%     \caption{Fixed-policy certificate tightening as the sampled support expands.}
%     \label{fig:fixed-policy-convergance}
% \end{wrapfigure}

We first evaluate the fixed-policy guarantees of Theorem~\ref{thm:bounds} by
progressively expanding the sampled tree of a fixed policy. Figure~\ref{fig:fixed-policy-convergance}
plots the projected robust value and its deterministic certificate against the
fraction of the complete depth-$H$ tree represented, with the exact robust
value obtained from the full tree shown for reference. Throughout the search,
the exact value remains within the certificate, while the interval contracts
as the sampled support expands and the projected value converges to the exact
value. The experiment is repeated over 50 independent nested tree expansions,
illustrating both the anytime validity and the convergence of the certificate.

\begin{figure}[h]
    \centering
    \begin{subfigure}[t]{0.49\linewidth}
        \centering
        \includegraphics[width=\linewidth]{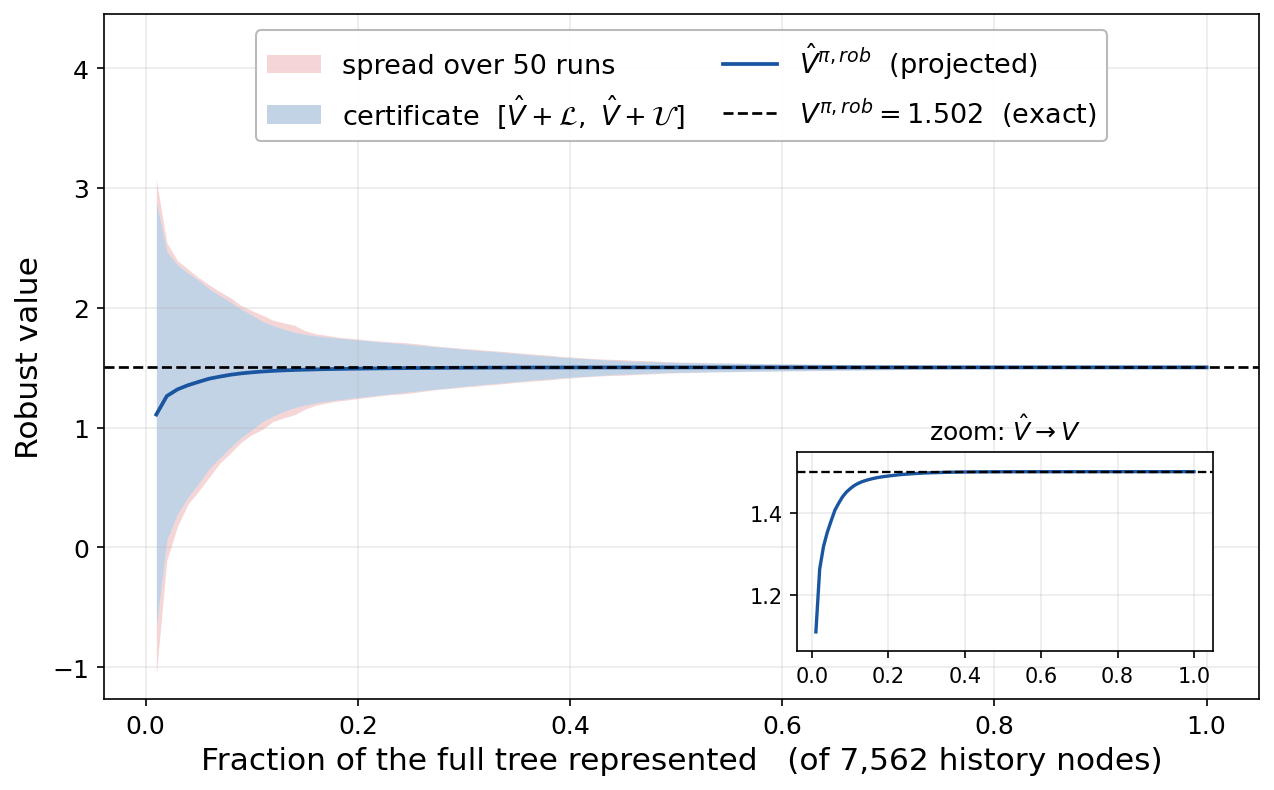}
        \caption{Fixed-policy certificate tightening.}
        \label{fig:fixed-policy-convergance}
    \end{subfigure}
    \hfill
    \begin{subfigure}[t]{0.49\linewidth}
        \centering
        \includegraphics[width=\linewidth]{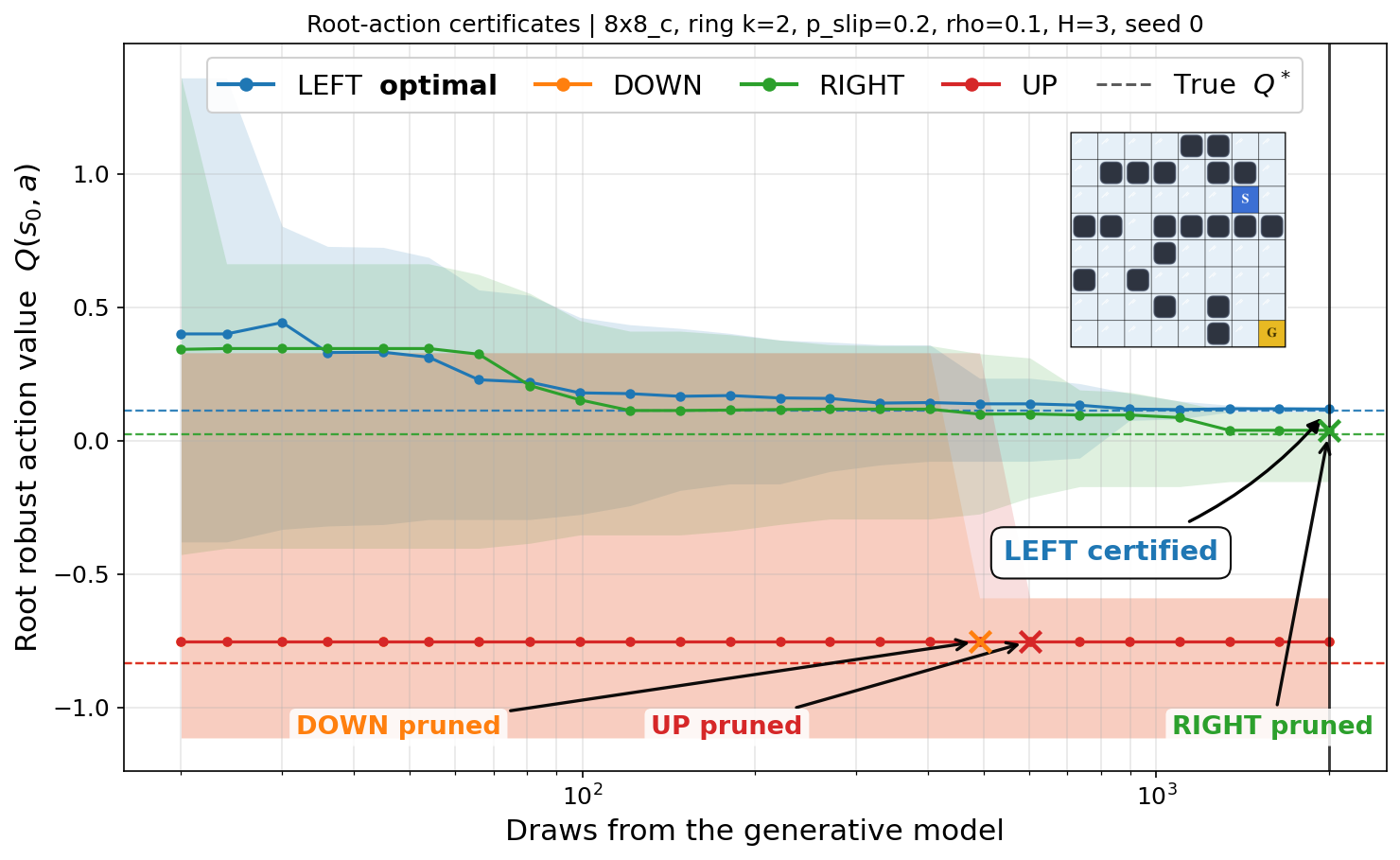}
        \caption{Pruning and root-action certification.}
        \label{fig:root-cert}
    \end{subfigure}
    \vspace{-5pt}
    \caption{Deterministic certificates: fixed-policy evaluation and CAR-MCTS planning.}
    % \caption{Deterministic certificates in CAR-MCTS.}
    \label{fig:certificate-results}
\end{figure}

% % ==================================
% % Figure - Map
% % ==================================
% \begin{wrapfigure}[7]{r}{0.15\linewidth}
%     \centering
%     \includegraphics[width=\linewidth]{Figures/map2.png}
% \end{wrapfigure}

% \paragraph{Pruning and action certification.}
Figure~\ref{fig:root-cert} illustrates how CAR-MCTS uses the certificates
during planning. Actions \textsc{down} and \textsc{up}, which lead toward
nearby holes, are quickly identified as suboptimal and pruned; their
certificates therefore stop changing once their subtrees are no longer
sampled. The remaining computation is concentrated on the \textsc{left} and
\textsc{right} actions. Since tree traversal selects the action with the largest
$\mathcal{URDB}$, the more promising \textsc{left} branch is sampled more
frequently, causing its certificate and projected value to tighten faster.
Eventually, \textsc{left} is deterministically certified as the optimal
action. Even when certification is not reached within the budget, this
targeted allocation improves the projected values used by CAR-MCTS's
$\arg\max_a \widehat Q^{\mathrm{rob}}(s_0,a)$ fallback decision.

% ==================================
% Figure - Root Certificate
% ==================================
% \begin{wrapfigure}[10]{r}{0.6\linewidth}
%     \centering
%     % \includegraphics[width=\linewidth]
%     \includegraphics[
%         width=\linewidth,
%         height=8cm,
%         keepaspectratio
%     ]{Figures/compare_k8_pexec0.05_n20_r_h.png}
%     \caption{comparison}
%     \label{fig:comparison}
% \end{wrapfigure}

% \begin{wrapfigure}[7]{R}{0.55\linewidth} \centering \includegraphics[width=\linewidth]{Figures/compare_k8_pexec0.05_n20_r_h.png} \caption{comparison} \label{fig:comparison} \end{wrapfigure}
% 
\textbf{Comparison with robust planning baselines.}
Figure~\ref{fig:comparison} compares CAR-MCTS with nominal UCT, RSS
\citep{shazman2026online}, and Robust Power UCT
\citep{dam2026online} on a $16\times16$ Frozen Lake domain. The execution
dynamics are fixed at $p_{\mathrm{true}}=0.05$, while $p_{\mathrm{slip}}$
denotes the slip probability assumed by the planner and the uncertainty radius
is adjusted to the resulting model mismatch. Hence, larger values of
$p_{\mathrm{slip}}$ primarily increase model mismatch rather than environment difficulty. CAR-MCTS maintains substantially
higher return over increasing mismatch than all baselines, with the advantage
becoming particularly pronounced at larger horizons. At $H=11$, CAR-MCTS
maintains near-maximal return through $p_{\mathrm{slip}}=0.5$, whereas nominal
UCT and Robust Power UCT deteriorate sharply and show no comparable benefit
from deeper lookahead.
\begin{wrapfigure}[11]{r}{0.54\linewidth} \centering \includegraphics[width=\linewidth]{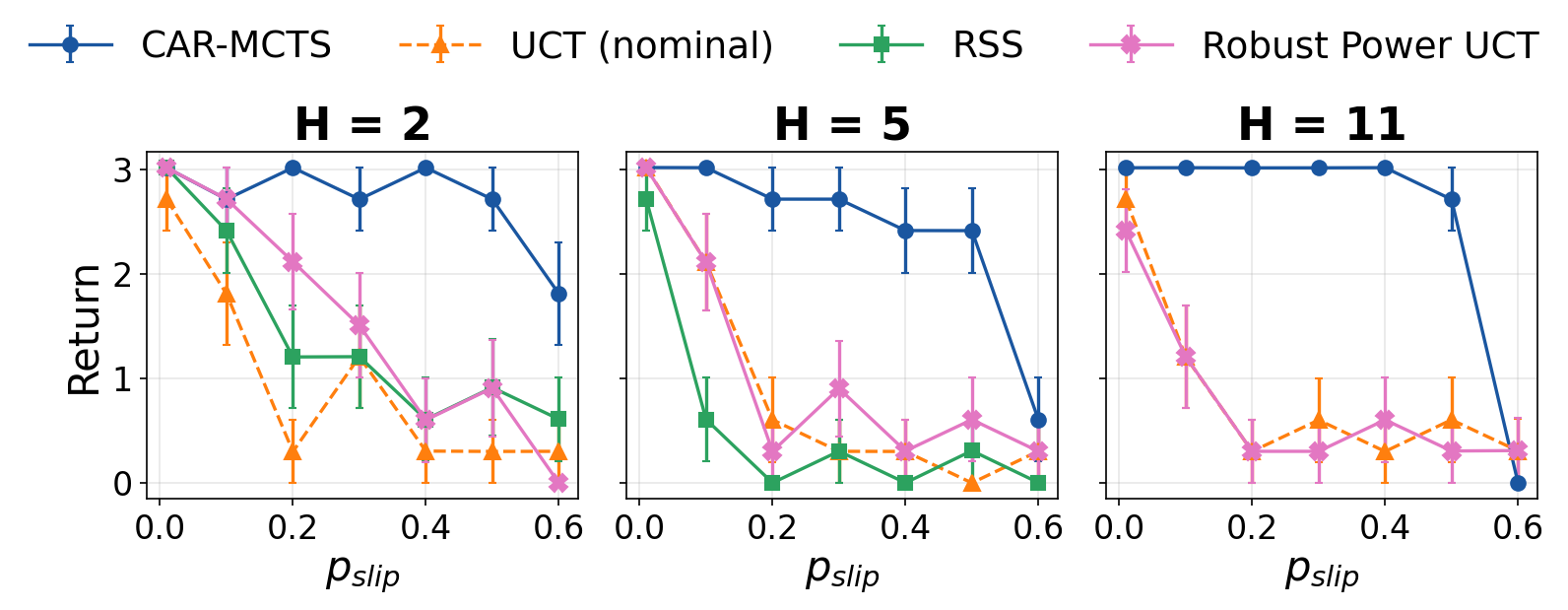}
\vspace{-17pt}
\caption{CAR-MCTS vs.\ robust planning baselines across model mismatch and planning horizons.} 
\label{fig:comparison}
\end{wrapfigure}
In particular, the degradation of nominal UCT
illustrates the limitation of planning with an inaccurate nominal model when
transition uncertainty is significant. 
RSS is omitted at $H=11$ because its sparse-sampling expansion becomes computationally prohibitive. 
At $p_{\mathrm{slip}}=0.6$, CAR-MCTS also deteriorates as the large model
mismatch induces a correspondingly large uncertainty set, leading to increasingly conservative behavior under the robust policy.

\textbf{Scalability with state-space size.}
Figure~\ref{fig:scaling} isolates the dependence of robust planning on
state-space cardinality. For each grid size,
the transition kernel is constructed so that every state has nonzero
probability under every transition row, i.e.,
$|\operatorname{supp} P^T(\cdot\mid s,a)|=|\Sfull|$, while the planning
horizon and simulation budget are held fixed. 
\begin{wrapfigure}[16]{r}{0.4\linewidth}
    \centering
    \includegraphics[width=\linewidth]{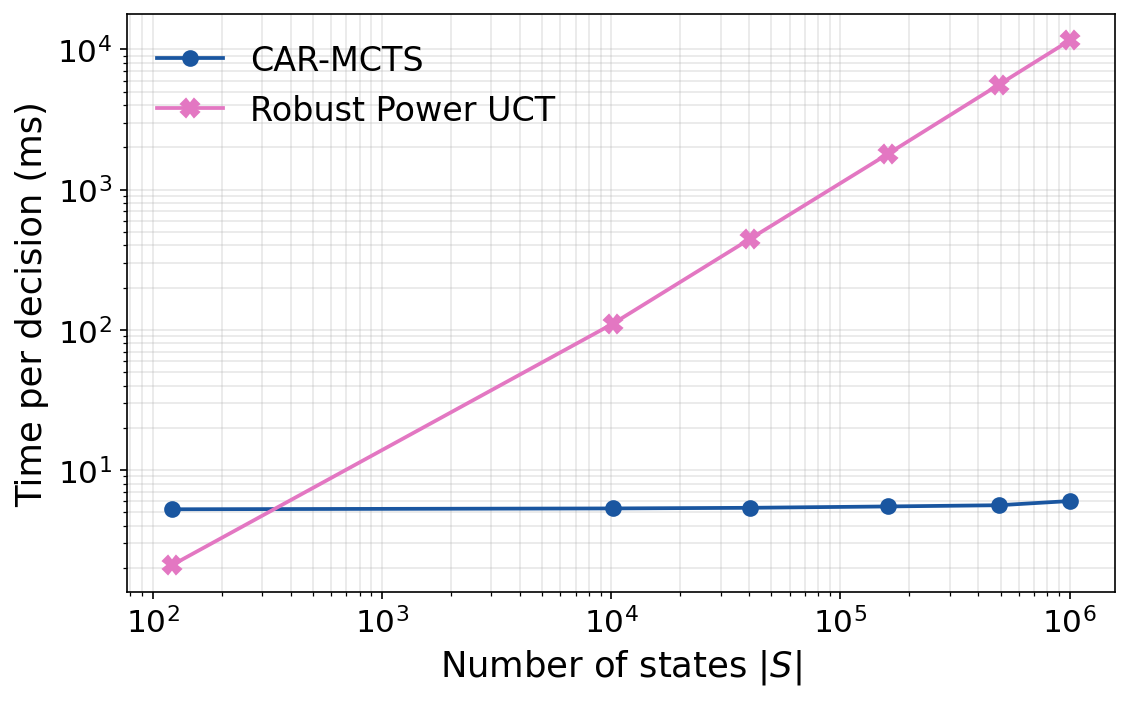}
    \vspace{-17pt}
    \caption{Decision time of CAR-MCTS and Robust Power UCT as $|\Sfull|$ grows under full-support transition kernels.}
    \label{fig:scaling}
\end{wrapfigure}
CAR-MCTS remains nearly constant at $5$--$6$ ms per decision as $|\Sfull|$ grows from $10^2$ to over $10^6$. In contrast, Robust Power UCT
grows from $2.1$ ms to $11.6$ s. This difference stems from the robust backups: CAR-MCTS operates only on the sampled support generated by its fixed simulation budget, 
whereas Robust Power UCT incurs computation over the
full state space, making online planning increasingly impractical as $|\Sfull|$ grows. Thus, CAR-MCTS scales essentially independently of full state-space cardinality, maintaining practical decision times even in very large state spaces. The same distinction is reflected in memory usage: CAR-MCTS requires only the
sampled-support representation and remains at approximately $0.2$ MB across
the sweep, whereas Robust Power UCT's full-state representation grows with
$|\Sfull|$, reaching $152$ MB at $|\Sfull|\approx10^6$.

\vspace{-2pt}
Additional ablations on robustness to increasing environment
slipperiness and model mismatch, sensitivity to the number of feasible successor states, scalability with planning depth, true optimal-action selection rate, and practical
certification and pruning rates are provided
in Appendix~\ref{exp:experiments}.
\vspace{-5pt}
\section{Conclusion}
\vspace{-5pt}
In this work, we presented a methodology for scalable robust online MDP
planning with deterministic anytime guarantees. The core principle we propose is to restrict
computation to the sampled search tree while explicitly accounting for the
contribution of unexplored outcomes, thereby avoiding enumeration of the full
state space. We first derived deterministic lower and upper bounds on the
error between the simplified and full robust values of an arbitrary policy
under general transition uncertainty, including non-rectangular sets, and showed that these certificates
tighten as the sampled support expands and converge as the full support is
recovered. Under rectangular uncertainty, we extended this construction to
recursive action-value certificates that bound the optimal robust value,
enabling provably suboptimal actions to be pruned, reducing unnecessary computation, and optimal root actions to be certified for early termination. Building on these results, we introduced
CAR-MCTS, which integrates the certificates directly into MCTS to guide
exploration, prune resolved branches, and terminate once the root decision is
certified, while retaining deterministic action-value certificates when the
sampling budget is exhausted. Experiments validate the derived guarantees,
demonstrate robustness to model mismatch, and show that CAR-MCTS remains
practical as the full state space grows while competing robust MCTS methods
become increasingly costly. To our knowledge, this is the first robust online
MCTS framework combining scalable sampled-support planning with deterministic
anytime guarantees.

% \TM{
% dams code:\\
% 1. they had a bug in the robust value backup\\
% 2. they say in the paper the root action is picked by argmax Q, but in the code the use max-visits.

% Tamir:\\
% limitations:\\
% 1. reward must be non negative\\
% 2. must have a zero reward state $s_f$\\
% 3. $s_f$ must exists in every model in the uncertainty set not just the nominal
% }

%% file: sections/Acknowledgments.tex
\section{Acknowledgments}
This work was supported by the Israel Ministry of Innovation, Science and Technology.

%% file: sections/appendix.tex
\appendix
\section{Mathematical analysis and proofs}

% ==================================
% Theorem 1
% ==================================
\input{Appendix/theorem_1}

\input{Appendix/lemma_1_value_function}
\input{Appendix/lemma_2_minmax}

\input{Appendix/lemma_3_Rmap}% Rmap
\input{Appendix/lemma_4_minMin}

\input{Appendix/corollary_1_convergance_bound}

% % ==================================
% % Theorem 2
% % ==================================
\clearpage
\input{Appendix/theorem_2}
\input{Appendix/Lemma_5_RUDB}

\clearpage
\input{Appendix/corollary_2_convergance_optimal}

% % ==================================
% % Experiments
% % ==================================
\clearpage
\input{Appendix/complexity}

\input{Appendix/experiments}

%% file: Appendix/theorem_1.tex
\subsection{Proof of Theorem \ref{thm:bounds} (Fixed-policy deterministic bounds):} \label{thm:thm_1_proof}

\textbf{Essentials lemmas and definitions for the proof:}

\bigskip \noindent \textbf{1.} The restriction mapping $\mathcal{R}$:
\begin{align}\label{RTmap}
    \RTseq[t][k-1](\PTseq[t][k-1];h_{k-1}) &=
    \begin{cases}
    \PTseq[t][k-1], & s_{t:k-1} \in \Sin(h_{k-1}), 
    \\
    0, & \text{otherwise}.
    \end{cases} 
\end{align}

\begin{lemmarestated}{lem:value function derivation}
Let H be the horizon of the MDP, $\pi=\{\pi_k\}_{k=t}^{H-1}$ be the policy,$r:\mathcal S\times\mathcal A\to[r_{\min},r_{\max}]$ be the bounded
state-action reward function , and $\PThk$ be the projected transition model defined in \eqref{eq:projected-transition-model}. The theoretical and simplified value functions $V_t^{\pi}$ and $\widehat{V}_t^{\pi}$ at time step $t$ are: 
\begin{align}
    &V_{t}^{\pi }( h_{t}) 
    =
    \sum_{j=t}^{H-1}
    \sum_{s_{t+1:j}}
    \prod_{k=t}^{j-1}
    \PTk\bigl(s_{k+1}\mid s_k,\pi_k(h_k)\bigr) 
    r\bigl(s_j,\pi_j(h_j)\bigr), \label{eq:thm_1_0}
    \\
    &\widehat{V}_t^{\pi}(h_t)
    =
    \sum_{j=t}^{H-1}
    \sum_{s_{t+1:j} \in \Sin(h_j)}
    \prod_{k=t}^{j-1}
    \PThk\bigl(s_{k+1}\mid s_k,\pi_k(h_k)\bigr) 
    r\bigl(s_j,\pi_j(h_j)\bigr).\label{eq:thm_1_1}
\end{align}
\end{lemmarestated}

\begin{lemmarestated}{lem:minmax}
For any 2 functions $f(\theta)$ and $\widehat{f}(\theta)$, the following holds:
\begin{align}
\inf_{\theta} f(\theta)
-
\inf_{{\theta}} \hat{f}(\theta)
\le
\sup_{\theta}
\Bigl(
f(\theta) - \hat{f}({\theta})
\Bigr).
\end{align}
\end{lemmarestated}

\begin{lemmarestated}{lem:Rmap}
    Let $\mathcal{P}$ denote some uncertainty set, and $\widehat{\mathcal{P}}$ denote its projected counterpart defined in \eqref{eq:projected_uncertainty_set_NR_transition}. Let $\mathcal{R}$ be the restriction map that takes a full probability distribution $P \in\mathcal{P}$ and retains only the entries associated with the sampled support:
\begin{align}
    \mathcal{R}(P;\Sin) &=
    \begin{cases}
    P(s), & s \in \Sin, 
    \\
    0, & \text{otherwise}.
    \end{cases}
\end{align}

\noindent Then minimizing some function $f : \widehat{P} \rightarrow \mathbb{R}$ over $\widehat{\mathcal{P}}$ is the same as minimizing $f(\mathcal{R}(P))$ over the full set $\mathcal{P}$, namely: 
\begin{align}
    \inf_{\substack{
    P \in \mathcal{P}
    }}
    {f}(\mathcal{R}(P))
    =
    \inf_{\substack{
    \widehat{P} \in \widehat{\mathcal{P}}
    }}
    {f}(\widehat{P}).
\end{align}
\end{lemmarestated}

\begin{lemmarestated}{lem:minMin}
For any 2 functions $f(\theta)$ and $\widehat{f}(\theta)$, the following holds:
\begin{align}
\inf_{\theta}
\left(
f(\theta) - \hat{f}({\theta})
\right).
\leq
\inf_{\theta} f(\theta)
-
\inf_{{\theta}} \hat{f}(\theta)
\end{align}
\end{lemmarestated}

\smallskip
\noindent\hfill\rule{0.5\linewidth}{0.4pt}\hfill
\smallskip

\textbf{Proof of Theorem~\ref{thm:bounds}.}
% ==================================
% Theorem restatement
% ==================================
\begin{theoremrestated}{thm:bounds}[Fixed-policy deterministic bounds] 
% Let $h_t$ be the history at time $t$, and $V_{t}^{\pi,\mathrm{rob}}(h_t)$ and $\widehat V_t^{\pi,\mathrm{rob}}(h_t)$ be the theoretical and simplified (\ref{eq:simplified-robust-value-function-non-stationary}) robust value functions that are achieved by following the same policy $\pi$, and $\UTseq, \UThseq$ be the original and projected, non-rectangular transition uncertainty set \ref{eq:uncertainty_set_NR_transition},\ref{eq:projected_uncertainty_set_NR_transition}. For any policy $\pi$ and horizon $H$, the difference between the theoretical and  simplified robust value functions is bounded deterministically by:
Let $h_t$ be the history at time $t$, and $V_{t}^{\pi,\mathrm{rob}}(h_t)$ and $\widehat V_t^{\pi,\mathrm{rob}}(h_t)$ denote the robust values of the same policy \(\pi\) in the full \eqref{eq:robust value function non-stationary} and simplified \eqref{eq:simplified-robust-value-function-non-stationary} RMDPs, respectively. For any transition uncertainty set $\UTseq$ and its projected counterpart $\UThseq$:
\begin{align}
    \LB
    \leq
     V_t^{\pi,\mathrm{rob}}(h_t) - \widehat V_t^{\pi,\mathrm{rob}}(h_t)
    \ \leq \
    \UB,
\end{align}
\noindent such that:
\begin{align}
    % &\UB
    % \triangleq
    % r_{max}
    % \Biggl (
    % H-t
    % -
    % \inf_{\PThseq \in \UThseq}
    % m_t^{\pi}\left(\PThseq \right) 
    % \Biggr)
    % \\[1em]
    &\UB
    \triangleq
    \begin{cases}
        r_{max}
        \biggl (
        H-t-
        \inf\limits_{\PThseq \in \UThseq}
        m_t^{\pi}\left(\PThseq \right)
        \biggr)
        , r_{max}\geq 0, 
        \\
        r_{max}
        \biggl (
        H-t-
        \sup\limits_{\PThseq \in \UThseq}
        m_t^{\pi}\left(\PThseq \right)
        \biggr)
        , r_{max} < 0,
    \end{cases}
    \\[1em]
    &\LB
    \triangleq
    \begin{cases}
        r_{min}
        \biggl (
        H-t-
        \sup\limits_{\PThseq \in \UThseq}
        m_t^{\pi}\left(\PThseq \right)
        \biggr)
        , r_{min}\geq 0, 
        \\
        r_{min}
        \biggl (
        H-t-
        \inf\limits_{\PThseq \in \UThseq}
        m_t^{\pi}\left(\PThseq \right)
        \biggr)
        , r_{min} < 0,
    \end{cases}
    % \\[1em]
    % &m_t^{\pi}\left(\PThseq \right)
    % \triangleq
    % \sum_{j=t}^{H-1}
    % \sum_{s_{t+1:j} \in \Sin(h_j)}
    % \prod_{k=t}^{j-1}
    % \PThk\bigl(s_{k+1}\mid s_k,\pi_k(h_k)\bigr)
\end{align}
and
\begin{align}
    m_t^{\pi}\left(\PThseq \right)
    \triangleq
    \sum_{j=t}^{H-1}
    \sum_{s_{t+1:j} \in \Sin(h_j)}
    \prod_{k=t}^{j-1}
    \PThk\bigl(s_{k+1}\mid s_k,\pi_k(h_k)\bigr).
\end{align}
% \TM{
% we can avoid the split for the lower bound $\LB$ if we put the $r_{min}$ inside the inf/sup. the motivation is algorithmic - reuse $m$ when $r_{min}<0$
% }
\end{theoremrestated}

\begin{proof}
% \noindent \\Using the derivation above (\ref{eq:thm_1_0}), (\ref{eq:thm_1_1}) for the theorectical/projected value function: 

% ==================================
% Proof
% ==================================
\noindent \\ The robust non-rectangular value function is defined as:
\begin{align}
    &V_t^{\pi,\mathrm{rob}}(h_t)
    \triangleq
    \inf_{\PTseq \in \UTseq} \
    V_t^{\pi} \left(h_t;\PTseq \right),
\end{align}

\noindent and the simplified robust value function:
\begin{align}
    \widehat{V}_t^{\pi,\mathrm{rob}}(h_t)
    \triangleq
    \inf_{\PThseq \in \UThseq}
    \, 
    \widehat{V}_t^{\pi} \left(h_t; \PThseq \right),    
\end{align}

\noindent where the theoretical and simplified value functions are defined in \eqref{eq:thm_1_0} and \eqref{eq:thm_1_1}. 
\bigskip
\noindent \\ Using Lemma \eqref{lem:Rmap} together with the restriction map $\RTseq[t][H-1]$ from \eqref{RTmap}, which project the full uncertainty set onto the projected uncertainty set, we obtain:
\begin{align}
    \widehat{V}_t^{\pi,\mathrm{rob}}(h_t) =
    \inf_{\PTseq \in \UTseq}
    \, 
    \widehat{V}_t^{\pi}\!\left(h_t; \RTseq[t][H-1](\PTseq[t][H-1])\right).
\end{align}

% =========================
% Upper Bound
% =========================
\noindent We first begin with the \textbf{Upper Bound}:
\begin{align}
    &V_t^{\pi,\mathrm{rob}}(h_t) 
    - 
    \widehat{V}_t^{\pi,\mathrm{rob}}(h_t)
    \nonumber \\[0.5em]
    =
    &\inf_{\PTseq \in \UTseq}
    \, V_t^{\pi}\!\left(h_t;\,\PTseq \right)
    -
    \inf_{\PTseq \in \UTseq}
    \, 
    \widehat{V}_t^{\pi}\!\left(h_t;\,
    \RTseq[t][H-1](\PTseq)\right)
    \\[0.5em]
    \leq
    &\sup_{\PTseq \in \UTseq}
    \left(
    \, V_t^{\pi}\!\left(h_t;\,\PTseq \right)
    -
    \widehat{V}_t^{\pi}\!\left(h_t;\,
    \RTseq[t][H-1](\PTseq)\right)
    \right)
\end{align}

\noindent For the last transition we used Lemma \ref{lem:minmax}.

\noindent Writing the value functions explicitly:
\begin{align}
    \sup_{\PTseq \in \UTseq}
    \Biggl(
    &\sum_{j=t}^{H-1}
    \sum_{s_{t+1:j}}
    \prod_{k=t}^{j-1}
    \PTk\bigl(s_{k+1}\mid s_k,\pi_k(h_k)\bigr) 
    r\bigl(s_j,\pi_j(h_j)\bigr)
    -
    \nonumber \\ 
    &\sum_{j=t}^{H-1}
    \sum_{s_{t+1:j} \in \Sin(h_j)}
    \prod_{k=t}^{j-1}
    \RT_k\left(\PTk\bigl(s_{k+1}\mid s_k,\pi_k(h_k)\bigr) \right)
    r\bigl(s_j,\pi_j(h_j)\bigr)
    \Biggr)
\end{align}

\noindent splitting the sums into in/out support: 
\begin{align} \label{eq:thm_1_inside_cancelation}
    \sup_{\PTseq \in \UTseq}
    \Biggl(
    &\sum_{j=t}^{H-1}
    \Bigl (
    \sum_{s_{t+1:j} \in \Sin(h_j)}
    \prod_{k=t}^{j-1}
    \underbrace{
    \left(
    \PTk\bigl(s_{k+1}\mid s_k,\pi_k(h_k)\bigr) 
    -
    \RT_k\left(\PTk\bigl(s_{k+1}\mid s_k,\pi_k(h_k)\bigr) \right)
    \right)}_{\equiv 0}
    r\bigl(s_j,\pi_j(h_j)\bigr)
    \nonumber \\ 
    &+
    \sum_{s_{t+1:j} \notin \Sin(h_j)}
    \prod_{k=t}^{j-1}
    \PTk\bigl(s_{k+1}\mid s_k,\pi_k(h_k)\bigr) 
    r\bigl(s_j,\pi_j(h_j)\bigr)
    \Bigr)
    \Biggr) 
    =
\end{align}

\noindent The under braced term equals zero by the definition of the Restricted mapping $\mathcal{R}$. Therefore:
\begin{align}
    &\sup_{\PTseq \in \UTseq}
    \left(
    \sum_{j=t}^{H-1}
    \sum_{s_{t+1:j} \notin \Sin(h_j)}
    \prod_{k=t}^{j-1}
    \PTk\bigl(s_{k+1}\mid s_k,\pi_k(h_k)\bigr) 
    r\bigl(s_j,\pi_j(h_j)\bigr)
    \right) \label{eq:thm_1_bound_Sout_true_reward}
    \\
    \leq
    &r_{max}
    \sup_{\PTseq \in \UTseq}
    \left(
    \sum_{j=t}^{H-1}
    \sum_{s_{t+1:j} \notin \Sin(h_j)}
    \prod_{k=t}^{j-1}
    \PTk\bigl(s_{k+1}\mid s_k,\pi_k(h_k)\bigr) 
    \right) \label{eq:thm_1_bound_Sout}
\end{align}

Since $r(s,a)\le r_{\max}$, the out-of-support contribution is upper bounded
by $r_{\max}$ times its probability mass. For $r_{\max}\ge 0$, maximizing
this product is equivalent to maximizing the out-of-support mass, yielding
the supremum in \eqref{eq:thm_1_bound_Sout}. If $r_{\max}<0$, the supremum is replaced by an infimum.

Since $P(s_{t+1:j}|h_t)$ defines a full probability distribution over sequences of states trajectories, we have:
\begin{align}
    \sum_{s_{t+1:j}}
    P(s_{t+1:j}|h_t)
    =
    \sum_{s_{t+1:j}}
    \prod_{k=t}^{j-1}
    \PTk\bigl(s_{k+1}\mid s_k,\pi_k(h_k)\bigr)
    \equiv
    1,
\end{align}

\noindent therefore:
\begin{align}
    \sum_{s_{t+1:j} \notin \Sin(h_j)}
    P(s_{t+1:j}|h_t)
    =
    1
    -
    \sum_{s_{t+1:j} \in \Sin(h_j)}
    P(s_{t+1:j}|h_t).
\end{align}

\noindent Thus \eqref{eq:thm_1_bound_Sout} equals:
\begin{align}
    &r_{max}
    \sup_{\PTseq \in \UTseq}
    \left(
    \sum_{j=t}^{H-1}
    \Biggl (
    1
    -
    \sum_{s_{t+1:j} \in \Sin(h_j)}
    \prod_{k=t}^{j-1}
    \PTk\bigl(s_{k+1}\mid s_k,\pi_k(h_k)\bigr) 
    \Biggr)
    \right)
    =
    \\[1em]
    &r_{max}
    \sup_{\PTseq \in \UTseq}
    \left(
    \sum_{j=t}^{H-1}
    \Biggl (
    1
    -
    \sum_{s_{t+1:j} \in \Sin(h_j)}
    \prod_{k=t}^{j-1}
    \RT_k\left(\PTk\bigl(s_{k+1}\mid s_k,\pi_k(h_k)\bigr) \right) 
    \Biggr)
    \right)
    =
\end{align}

\noindent Using Lemma \ref{lem:Rmap} again:
\begin{align}
    &r_{max}
    \sup_{\PThseq \in \UThseq}
    \left(
    \sum_{j=t}^{H-1}
    \Biggl (
    1
    -
    \sum_{s_{t+1:j} \in \Sin(h_j)}
    \prod_{k=t}^{j-1}
    \PThk\bigl(s_{k+1}\mid s_k,\pi_k(h_k)\bigr) 
    \Biggr)
    \right)
    =
    \\[1em]
    &r_{max}
    \Biggl (
    H-t
    -
    \inf_{\PThseq \in \UThseq}
    \sum_{j=t}^{H-1}
    \sum_{s_{t+1:j} \in \Sin(h_j)}
    \prod_{k=t}^{j-1}
    \PThk\bigl(s_{k+1}\mid s_k,\pi_k(h_k)\bigr) 
    \Biggr)
    \triangleq
    \\[1em]
    &\UB
    \label{thm_1_final_UB}
\end{align}

% =========================
% Lower Bound
% =========================
\noindent \textbf{Lower Bound:}
\begin{align}
    &V_t^{\pi,\mathrm{rob}}(h_t) 
    - 
    \widehat{V}_t^{\pi,\mathrm{rob}}(h_t)
    =
    \nonumber \\[0.5em]
    &\inf_{\PTseq \in \UTseq}
    \, V_t^{\pi}\!\left(h_t;\,\PTseq \right)
    -
    \inf_{\PTseq \in \UTseq}
    \, 
    \widehat{V}_t^{\pi}\!\left(h_t;\,
    \RTseq[t][H-1](\PTseq)\right)
    \geq
    \\[0.5em]
    &\inf_{\PTseq \in \UTseq}
    \left(
    \, V_t^{\pi}\!\left(h_t;\,\PTseq \right)
    -
    \widehat{V}_t^{\pi}\!\left(h_t;\,
    \RTseq[t][H-1](\PTseq)\right)
    \right)
    =
\end{align}

\noindent For the last transition we used Lemma \ref{lem:minMin}. As shown above in \eqref{eq:thm_1_inside_cancelation}, the difference on the support is canceled, and we get:
\begin{align}
    &\inf_{\PTseq \in \UTseq}
    \sum_{j=t}^{H-1}
    \Biggl (
    \sum_{s_{t+1:j} \notin \Sin(h_j)}
    \prod_{k=t}^{j-1}
    \PTk\bigl(s_{k+1}\mid s_k,\pi_k(h_k)\bigr) 
    r\bigl(s_j,\pi_j(h_j)\bigr)
    \Biggr) 
    \geq
\end{align}

% \noindnet we now differentiate between 2 scenarios: $r_{min}\geq 0$ and $r_{min} < 0$,  where $r_{min} \leq r(s,a), \ \ \forall(s,a) \in (\Sfull \times \Afull)$:

% \noindent 1) $r_{min}\geq 0$:
% % 
% \begin{align}
%     &r_{min}
%     \inf_{\PTseq \in \UTseq}
%     \sum_{j=t}^{H-1}
%     \Biggl (
%     \sum_{s_{t+1:j} \notin \Sin(h_j)}
%     \prod_{k=t}^{j-1}
%     \PTk\bigl(s_{k+1}\mid s_k,\pi_k(h_k)\bigr) 
%     \Biggr)
%     =
%     \\[0.5em]
%     &r_{min}
%     \inf_{\PThseq \in \UThseq}
%     \sum_{j=t}^{H-1}
%     \Biggl (
%     1-
%     \sum_{s_{t+1:j} \in \Sin(h_j)}
%     \prod_{k=t}^{j-1}
%     \PThk\bigl(s_{k+1}\mid s_k,\pi_k(h_k)\bigr) 
%     \Biggr)
%     \\[0.5em]
%     =
%     &r_{min}
%     \Biggl (
%     H-t-
%     \sup_{\PThseq \in \UThseq}
%     \sum_{j=t}^{H-1}
%     \sum_{s_{t+1:j} \in \Sin(h_j)}
%     \prod_{k=t}^{j-1}
%     \PThk\bigl(s_{k+1}\mid s_k,\pi_k(h_k)\bigr) 
%     \Biggr)
% \end{align}

% \noindent 2) $r_{min} < 0$:
% 

Since $r(s,a)\ge r_{\min}$, the out-of-support contribution is lower bounded
by $r_{\min}$ times its probability mass. For $r_{\min}\leq 0$, minimizing
this product is equivalent to maximizing the out-of-support mass, yielding the
supremum below. If $r_{\min}>0$, the supremum is replaced by an infimum.
\begin{align}
    &r_{min}
    \sup_{\PTseq \in \UTseq}
    \sum_{j=t}^{H-1}
    \Biggl (
    \sum_{s_{t+1:j} \notin \Sin(h_j)}
    \prod_{k=t}^{j-1}
    \PTk\bigl(s_{k+1}\mid s_k,\pi_k(h_k)\bigr) 
    \Biggr)
    =
    \\[0.5em]
    &r_{min}
    \sup_{\PThseq \in \UThseq}
    \sum_{j=t}^{H-1}
    \Biggl (
    1-
    \sum_{s_{t+1:j} \in \Sin(h_j)}
    \prod_{k=t}^{j-1}
    \PThk\bigl(s_{k+1}\mid s_k,\pi_k(h_k)\bigr) 
    \Biggr)
    =
    \\[0.5em]
    &r_{min}
    \Biggl (
    H-t-
    \inf_{\PThseq \in \UThseq}
    \sum_{j=t}^{H-1}
    \sum_{s_{t+1:j} \in \Sin(h_j)}
    \prod_{k=t}^{j-1}
    \PThk\bigl(s_{k+1}\mid s_k,\pi_k(h_k)\bigr) 
    \Biggr)
    \\[1em]
    &\triangleq
    \LB \label{eq:thm_1_final_LB}
\end{align}

% \noindent To conclude the lower bound for the 2 scenarios we have:
% % 
% \begin{align}
%     &\LB \label{eq:thm_1_final_LB}
%     =
%     \begin{cases}
%         r_{min}
%         \biggl (
%         H-t-
%         \sup\limits_{\PThseq \in \UThseq}
%         m_t^{\pi}\left(\PThseq \right)
%         \biggr)
%         , r_{min}\geq 0, 
%         \\
%         r_{min}
%         \biggl (
%         H-t-
%         \inf\limits_{\PThseq \in \UThseq}
%         m_t^{\pi}\left(\PThseq \right)
%         \biggr)
%         , r_{min} < 0
%     \end{cases}
%     \\[1em]
%     &m_t^{\pi}\left(\PThseq \right)
%     \triangleq
%     \sum_{j=t}^{H-1}
%     \sum_{s_{t+1:j} \in \Sin(h_j)}
%     \prod_{k=t}^{j-1}
%     \PThk\bigl(s_{k+1}\mid s_k,\pi_k(h_k)\bigr) 
% \end{align}

\noindent which concludes the proof.
\end{proof}

\paragraph{Transition to a rectangular setting:}
\noindent \\ If we assume rectangular uncertainty sets \ref{eq:rect uncertainty set}, the above bounds \ref{thm_1_final_UB}, \ref{eq:thm_1_final_LB} can be written as nested optimizations, that allows dynamic programming: 
\begin{align} \label{eq:thm_1_UB_rectangular}
    \UBrec
    \triangleq 
    r_{max} %first
    &\Biggl(
    H -1-t
    -
    \Biggl (
    \inf_{\PThk[t] \in \UTh[t]}
    \sum _{s_{t+1} \in \Sin(h_{t+1}^-)}
    \PThk[t](s_{t+1}\mid \pi_t(h_t),s_t;h_t)\,
    \Biggl (
    \nonumber \\ \nonumber
    &1+
    \inf_{\PThk[t+1] \in \UTh[t+1]}
    \sum _{s_{t+2} \in \Sin(h_{t+2}^-)}
    \PThk[t+1](s_{t+2}\mid \pi_{t+1},s_{t+1})\,
    \Biggl (
    \\ \nonumber 
    &\cdots
    \\ 
    &1+
    \inf_{\PThk[H-2] \in \UTh[H-2]}
    \sum _{s_{H-1} \in \Sin(h_{H-1}^-)}
    \PThk[H-2](s_{H-1}\mid \pi_{H-2},s_{H-2})\,
    \Biggr )  
    % \Biggr )
    % \Biggr )
    % \Biggr )
    \cdots
    \Biggr )
\end{align}
\begin{align} \label{eq:thm_1_LB_rectangular}
    \LBrec
    \triangleq 
    r_{min} %first
    &\Biggl(
    H -1-t
    -
    \Biggl (
    \inf_{\PThk[t] \in \UTh[t]}
    \sum _{s_{t+1} \in \Sin(h_{t+1}^-)}
    \PThk[t](s_{t+1}\mid \pi_t(h_t),s_t;h_t)\,
    \Biggl (
    \nonumber \\ \nonumber
    &1+
    \inf_{\PThk[t+1] \in \UTh[t+1]}
    \sum _{s_{t+2} \in \Sin(h_{t+2}^-)}
    \PThk[t+1](s_{t+2}\mid \pi_{t+1},s_{t+1})\,
    \Biggl (
    \\ \nonumber 
    &\cdots
    \\ 
    &1+
    \inf_{\PThk[H-2] \in \UTh[H-2]}
    \sum _{s_{H-1} \in \Sin(h_{H-1}^-)}
    \PThk[H-2](s_{H-1}\mid \pi_{H-2},s_{H-2})\,
    \Biggr )  
    % \Biggr )
    % \Biggr )
    % \Biggr )
    \cdots
    \Biggr )
\end{align}

\noindent  We can write \ref{eq:thm_1_UB_rectangular}, \ref{eq:thm_1_LB_rectangular} in a compact form, recursively:
\begin{align} \label{eq:thm_1_rect_bounds}
    &\UBrec
    \triangleq
    r_{max}
    \Bigl (
    H -1- t
    -
    \delta_{t}^{\pi^U}(s_t|h_t)
    \Bigr )
    \\[1em]
    &\LBrec
    \triangleq
    r_{min}
    \Bigl (
    H -1- t
    -
    \delta_{t}^{\pi^L}(s_t|h_t)
    \Bigr )
\end{align}

\noindent Such that for $k \in [t, H-2]$:
\begin{align} 
    \delta_{k}^{\pi}(s_k|h_k) \label{eq:thm_1_delta}
    &=
    \inf_{\PThk \in \UTh[k]}
    \delta_{k}^{\pi, \PThk}(s_k, \pi_k(h_k))
    \\[1em]
    \delta_{k}^{\pi,\PThk}(s_k,a_k|h_k) \label{eq:thm_1_delta_a}
    &=
    \sum _{s_{k+1} \in \Sin(h_{k+1}^-)}
    \PThk(s_{k+1}\mid s_{k},a_{k})\,
    \left [
    1 
    +
    \delta_{k+1}^{\pi}(s_{k+1}|h_{k+1})
    \right ]
\end{align}

\noindent and for $k=H-1$ there is no continuation therefore:
\begin{align}
    \delta_{H-1}^{\pi}(s_{H-1})
    \equiv
    0
\end{align}

\bigskip \noindent \textbf{Note}:  we switch  the $inf$ to $sup$:
\begin{enumerate}
    \item In equations \ref{eq:thm_1_LB_rectangular}, \ref{eq:thm_1_delta}: if $r_{\min}>0$.
    \item In equations \ref{eq:thm_1_UB_rectangular}, \ref{eq:thm_1_delta}: if $r_{\max}<0$.
    \item $\delta$ is shared only when $r_{min}\leq 0\leq r_{max}$
\end{enumerate}

\smallskip
\noindent\rule{\linewidth}{0.4pt}
\smallskip

%% file: Appendix/lemma_1_value_function.tex
\subsection{Proof of Lemma \ref{lem:value function derivation}}

\begin{lemma}\label{lem:value function derivation}
Let H be the horizon of the MDP, $\pi=\{\pi_k\}_{k=t}^{H-1}$ be the policy,$r:\mathcal S\times\mathcal A\to[r_{\min},r_{\max}]$ be the bounded
state-action reward function , and $\PThk$ be the projected transition model defined in \eqref{eq:projected-transition-model}. The theoretical and simplified value functions $V_t^{\pi}$ and $\widehat{V}_t^{\pi}$ at time step $t$ are: 
\begin{align}
    &V_{t}^{\pi }( h_{t}) \label{eq:value function derived}
    =
    \sum_{j=t}^{H-1}
    \sum_{s_{t+1:j}}
    \prod_{k=t}^{j-1}
    \PTk\bigl(s_{k+1}\mid s_k,\pi_k(h_k)\bigr) 
    r\bigl(s_j,\pi_j(h_j)\bigr),
    \\
    &\widehat{V}_t^{\pi}(h_t)\label{eq:simp value function derived}
    =
    \sum_{j=t}^{H-1}
    \sum_{s_{t+1:j} \in \Sin(h_j)}
    \prod_{k=t}^{j-1}
    \PThk\bigl(s_{k+1}\mid s_k,\pi_k(h_k)\bigr) 
    r\bigl(s_j,\pi_j(h_j)\bigr).
\end{align}
\end{lemma}

% ==================================
% Proof
% ==================================
\begin{proof}
\noindent The value function is the expected future cumulative reward: 
\begin{align} \label{eq:value_supplemantry}
    V_{t}^{\pi }( h_{t}) =
    \mathbb{E}^{\pi, \PT }_{s_{t+1:H-1}}
    \left[ \sum _{j=t}^{H-1} r(s_{j} ,\pi_j) \mid h_t\right] .
\end{align}

\noindent By the linearity of the expectation $\mathbb{E}[\sum f(\cdot)]=\sum \mathbb{E}[f(\cdot)]$, and writing the expectation explicitly:
\begin{align} \label{eq:lem1_expectation}
    \sum _{j=t}^{H-1}
    \mathbb{E}^{\pi,\PT }_{s_{t+1:H-1}} \left[r(s_{j} ,\pi_j)\mid h_t \right]
    =
    \sum _{j=t}^{H-1}
    \sum _{s_{t+1:j}}
    P\bigl(s_{t+1:j}|h_t\bigr)
    r(s_{j} ,\pi_j).
\end{align}

\noindent Applying chain rule on the states sequence probability:
\begin{align}
    P(s_{t+1:j}|h_t)
    =
    \prod_{k=t}^{j-1}
    P(s_{k+1}|h_t, s_{t+1:k}).
\end{align}

\noindent We assume markovian models, therefore the probability of transition to next state depends only on current state and action. since $h_k=(s_{t:k}, a_{t:k-1})$, the states sequence probability is:
\begin{align} \label{eq:lem1_state_seq_prob}
    P(s_{t+1:j}|h_t)
    =
    \prod_{k=t}^{j-1}
    \PTk \bigl(s_{k+1}|s_k, \pi(h_k)\bigr).
\end{align}

\noindent Where in \eqref{eq:lem1_state_seq_prob} we assume non-stationary transition model. Using \eqref{eq:lem1_expectation} and \eqref{eq:lem1_state_seq_prob}, we derived the final expression for the value function:
\begin{align}
    V_{t}^{\pi }(h_{t})
    =
    \sum _{j=t}^{H-1}
    \sum _{s_{t+1:j}}
    \prod_{k=t}^{j-1}
    \PTk \bigl(s_{k+1}|s_k, \pi(h_k)\bigr)
    r\bigl(s_j,\pi_j(h_j)\bigr).
\end{align}

\noindent Similarly, using the projected transition model defined in \eqref{eq:projected-transition-model}, the simplified value function is:
\begin{align}
    \widehat{V}_t^{\pi}(h_t)
    =
    \sum_{j=t}^{H-1}
    \sum_{s_{t+1:j} \in \Sin(h_j)}
    \prod_{k=t}^{j-1}
    \PThk\bigl(s_{k+1}\mid s_k,\pi_k(h_k)\bigr) 
    r\bigl(s_j,\pi_j(h_j)\bigr).
\end{align}

\end{proof}

\noindent For \textbf{stochastic policies}:
\begin{align}
    V_{t}^{\pi }(h_{t})
    =
    \sum _{j=t}^{H-1}
    \sum_{a_{t:j}}
    \sum _{s_{t+1:j}}
    \left(
    \prod_{k=t}^{j-1}
    \pi(a_k|h_k)
    \PTk(s_{k+1}|s_k, a_k)
    \right)
    \pi(a_j|h_j)
    r(s_j,a_j).
\end{align}

\smallskip
\noindent\rule{\linewidth}{0.4pt}
\smallskip

%% file: Appendix/lemma_2_minmax.tex
\subsection{Proof of Lemma \ref{lem:minmax}} 

\begin{lemma}\label{lem:minmax}
for any 2 function $f(\theta), \widehat{f}(\theta)$, the following holds:
\begin{align}
\inf_{\theta} f(\theta)
-
\inf_{{\theta}} \hat{f}(\theta)
\le
\sup_{\theta}
\Bigl(
f(\theta) - \hat{f}({\theta})
\Bigr)
\end{align}
\end{lemma}

\begin{proof}
\noindent \\Let
\[
\varepsilon 
\triangleq
\sup_{\theta \in \Theta} 
\Bigl(f(\theta)-g(\theta)\Bigr)
\]

\noindent Then for every $\theta$,
\[
f(\theta) - g(\theta)
\leq
\varepsilon
\ \
\rightarrow
\ \
f(\theta)
\leq
\varepsilon
+ g(\theta)
\]

\noindent Taking minima gives
\[
\inf_{\theta} f(\theta)
\le
\inf_{\theta} g(\theta) + \varepsilon
\]

\noindent Therefore
\[
\inf_{\theta} f(\theta)
-
\inf_{\theta} g(\theta) 
\leq
\varepsilon
\triangleq
\sup_{\theta \in \Theta} 
\Bigl(f(\theta)-g(\theta)\Bigr)
\]

\end{proof}
% \end{lemma}

\smallskip
\noindent\rule{\linewidth}{0.4pt}
\smallskip

%% file: Appendix/lemma_3_Rmap.tex
\subsection{Proof of Lemma \ref{lem:Rmap}} 

\begin{lemma}\label{lem:Rmap}
Let $\mathcal{P}$ denote some uncertainty set, and $\widehat{\mathcal{P}}$ denote its projected version. Let $\mathcal{R}$ be the restriction map that takes a full probability distribution $P \in\mathcal{P}$ and retains only the entries associated with the sampled support:
\begin{align}
    \mathcal{R}(P;\Sin) &=
    \begin{cases}
    P(s), & s \in \Sin, 
    \\
    0, & \text{otherwise}.
    \end{cases}
\end{align}

% \VI{
% Can we relate this back to \eqref{eq:projected-transition-model} and \eqref{eq:projected-obervation-model}, even for a single time instant?}
\noindent Then minimizing some function $f : \widehat{P} \rightarrow \mathbb{R}$ over $\widehat{\mathcal{P}}$ is the same as minimizing $f(\mathcal{R}(P))$ over the full set $\mathcal{P}$, namely: 
\begin{align}
    \inf_{\substack{
    P \in \mathcal{P}
    }}
    {f}(\mathcal{R}(P))
    =
    \inf_{\substack{
    \widehat{P} \in \widehat{\mathcal{P}}
    }}
    {f}(\widehat{P}),
\end{align}
\end{lemma}

\begin{proof}
\noindent \\By the definition of the projected uncertainty sets as the images of the full uncertainty sets under the restriction maps, we have
\begin{align}
    \widehat{\mathcal P}
    =
    \mathcal R
    \left(
    \mathcal P
    \right),
\end{align}

\noindent We prove the desired equality by showing both inequalities. First, let
\begin{align}
    P \in \mathcal P,
\end{align}

\noindent be arbitrary full distribution. By definition:
\begin{align}
    \mathcal R(P)\in\widehat{\mathcal P},
\end{align}

\noindent Therefore,
\begin{align}
    \inf_{\widehat P\in\widehat{\mathcal P} }
    f(\widehat{P})
    \le
    f(\mathcal R(P)).
\end{align}

\noindent Since this holds for every
\begin{align}
    P \in\mathcal P,
\end{align}

\noindent we can take the infimum over the full uncertainty set and obtain
\begin{align}
    \inf_{\widehat P\in\widehat{\mathcal P} }
    f(\widehat{P})
    \le
    \inf_{ P\in \mathcal P }
    f(\mathcal R(P)).
\end{align}

\noindent For the opposite direction, let
\begin{align}
    \widehat P\in\widehat{\mathcal P},
\end{align}

\noindent be arbitrary projected distribution. Since the projected uncertainty set is defined as the image of the full uncertainty set, there exist a full model:
\begin{align}
    P \in \mathcal P,
\end{align}

\noindent such that
\begin{align}
    \widehat P
    =
    \mathcal R(P),
\end{align}

\noindent Therefore,
\begin{align}
    f(\widehat P)
    =
     f(\mathcal R(P))
\end{align}

\noindent Hence,
\begin{align}
    \inf_{ P\in \mathcal P }
    f(\mathcal R(P))
    \le
    f(\widehat{P}).
\end{align}

\noindent Since this holds for every
\begin{align}
    \widehat P\in\widehat{\mathcal P},
\end{align}

\noindent we can take the infimum over the projected uncertainty set and obtain

\begin{align}
    \inf_{ P\in \mathcal P }
    f(\mathcal R(P))
    \le
    \inf_{ \widehat P\in \widehat{\mathcal P} }
    f(\widehat{P}).
\end{align}

\noindent Combining the two inequalities gives equality:
\begin{align}
    \inf_{ P\in \mathcal P }
    f(\mathcal R(P))
    =
    \inf_{ \widehat P\in \widehat{\mathcal P} }
    f(\widehat{P}).
\end{align}

\noindent Which proves the Lemma.

\end{proof}
% \end{lemma}

\smallskip
\noindent\rule{\linewidth}{0.4pt}
\smallskip

%% file: Appendix/lemma_4_minMin.tex
\subsection{Proof of Lemma \ref{lem:minMin}} 

\begin{lemma}\label{lem:minMin}
for any 2 function $f(\theta), \widehat{f}(\theta)$, the following holds:
\begin{align}
\inf_{\theta}
\left(
f(\theta) - \hat{f}({\theta})
\right)
\leq
\inf_{\theta} f(\theta)
-
\inf_{{\theta}} \hat{f}(\theta)
\end{align}
\end{lemma}

\begin{proof}
\noindent \\Let
\[
\varepsilon 
\triangleq
\inf_{\theta \in \Theta} 
\Bigl(
f(\theta)-g(\theta)
\Bigr)
\]

\noindent Then for every $\theta$,
\[
\varepsilon
\leq
f(\theta) - g(\theta)
\ \
\rightarrow
\ \
\varepsilon
+ g(\theta)
\leq
f(\theta)
\]

\noindent Taking minima gives
\[
\varepsilon
+
\inf_{\theta} g(\theta)
\le
\inf_{\theta} f(\theta)
\]

\noindent Therefore
\[
\varepsilon
\triangleq
\inf_{\theta \in \Theta} 
\Bigl(f(\theta)-g(\theta)\Bigr)
\leq
\inf_{\theta} f(\theta)
-
\inf_{\theta} g(\theta) 
\]

\end{proof}
% \end{lemma}

\medskip
\noindent\rule{\linewidth}{0.4pt}
\medskip

%% file: Appendix/corollary_1_convergance_bound.tex
\subsection{Proof of Corollary \ref{cor:Bound-diminishing} (Asymptotic convergence)}\label{cor:cor_1_proof}

\begin{corollaryrestated}{cor:Bound-diminishing}[Asymptotic convergence] 
For any policy $\pi$, as the sampled state support expands to the full state space at 
every node $\Sin(h_k^-) \uparrow \Sfull, \ k=t,\ldots, H$, the simplified robust 
value \eqref{eq:simplified-robust-value-function-non-stationary} convergences to the true robust
value \eqref{eq:robust value function non-stationary}:
\begin{align}
    \widehat V_t^{\pi,\mathrm{rob}}(h_t)
    \xrightarrow{
    \substack{
    \Sin \rightarrow \Sfull
    }}
    V_t^{\pi,\mathrm{rob}}(h_t).
\end{align}
\end{corollaryrestated}

\begin{proof}
\noindent Restating the upper and lower bounds of theorem \ref{thm:bounds}:
\begin{align}
    % &\UB
    % \triangleq
    % r_{max}
    % \Biggl (
    % H-t
    % -
    % \inf_{\PThseq \in \UThseq}
    % m_t^{\pi}\left(\PThseq \right) 
    % \Biggr)
    % \\[1em]
    &\UB
    \triangleq
    \begin{cases}
        r_{max}
        \biggl (
        H-t-
        \inf\limits_{\PThseq \in \UThseq}
        m_t^{\pi}\left(\PThseq \right)
        \biggr)
        , r_{max}\geq 0, 
        \\
        r_{max}
        \biggl (
        H-t-
        \sup\limits_{\PThseq \in \UThseq}
        m_t^{\pi}\left(\PThseq \right)
        \biggr)
        , r_{max} < 0
    \end{cases}
    \\[1em]
    &\LB
    \triangleq
    \begin{cases}
        r_{min}
        \biggl (
        H-t-
        \sup\limits_{\PThseq \in \UThseq}
        m_t^{\pi}\left(\PThseq \right)
        \biggr)
        , r_{min}\geq 0, 
        \\
        r_{min}
        \biggl (
        H-t-
        \inf\limits_{\PThseq \in \UThseq}
        m_t^{\pi}\left(\PThseq \right)
        \biggr)
        , r_{min} < 0
    \end{cases}
    \\[1em]
    &m_t^{\pi}\left(\PThseq \right)
    \triangleq
    \sum_{j=t}^{H-1}
    \sum_{s_{t+1:j} \in \Sin(h_j)}
    \prod_{k=t}^{j-1}
    \PThk\bigl(s_{k+1}\mid s_k,\pi_k(h_k)\bigr)
\end{align}

\noindent The  state-trajectory likelihood is bounded by 1, and is monotonic non-decreasing for every new state sample that is added, and equals 1 over the full state space, namely:
\begin{align}
    \sum_{s_{t+1:j} \in \Sin(h_j)}
    P(s_{t+1:j}|h_t)
    =
    \sum_{s_{t+1:j} \in \Sin(h_j)}
    \prod_{k=t}^{j-1}
    \PTk\bigl(s_{k+1}\mid s_k,\pi_k(h_k)\bigr)
    \xrightarrow{
    \substack{
    \Sin \rightarrow \Sfull
    }}
    1,
\end{align}

\noindent Therefore
\begin{align}
    m_t^{\pi}\left(\PThseq \right)
    \xrightarrow{
    \substack{
    \Sin \rightarrow \Sfull
    }}
    H - t
\end{align}

\noindent Which induces:
\begin{align}
    \UB
    \xrightarrow{
    \substack{
    \Sin \rightarrow \Sfull
    }}
    0
    \\
    \LB
    \xrightarrow{
    \substack{
    \Sin \rightarrow \Sfull
    }}
    0
\end{align}

\end{proof}

\smallskip
\noindent\rule{\linewidth}{0.4pt}
\smallskip

%% file: Appendix/theorem_2.tex
\subsection{Proof of Theorem \ref{thm:optimal-bounds} (Optimal robust value bounds):} \label{thm:thm_2_proof}

% ==================================
% useful definitions
% ==================================
\textbf{Essentials lemmas and definitions for the proof: }
\\[1em] \noindent  \textbf{1.} In-support expectation operator:
\begin{align}
    \widehat{\mathbb{E}}_{s}
        \!\left[
        f(h)
        \right]
        =
        \sum_{s \in \Sin}
        \PTt(s)\cdot
        f(h)
\end{align}

\noindent  \textbf{2.} Upper robust deterministic bound: 
    \begin{align} \label{eq:urdba}
        \RUDBa
        &\triangleq
        r(s_t,a_t)
        +
        \inf_{\PTht \in \UTht}
        \Biggl(
        \widehat{\mathbb{E}}_{s_{t+1}|h_t}^{\PTht}
        \!\left[
        \widehat V_{t+1,rect}^{\pi^U,\widehat{P}_{t+1}^{T, \mathcal{URDB}}}(h_{t+1}) 
        \right]
        +
        \mathcal{U}_{t,rect}^{\pi^U,\PTht}(h_{t},a_{t})
        \Biggr)
        \\
        &=
        \widehat Q_{t,rect}^{\pi^U,\widehat{P}_{t}^{T, \mathcal{URDB}}}(h_{t},a_{t})
        +
        \mathcal{U}_{t,rect}^{\pi^U,\widehat{P}_{t}^{T, \mathcal{URDB}}}(h_{t},a_{t})
    \end{align}

\noindent \textbf{3.} Lower robust deterministic bound:
    \begin{align} \label{eq:lrdba}
        \RLDBa
        &\triangleq
        r(s_t,a_t)
        +
        \inf_{\PTht \in \UTht}
        \Biggl(
        \widehat{\mathbb{E}}_{s_{t+1}|h_t}^{\PTht}
        \!\left[
        \widehat V_{t+1,rect}^{\pi^L,\widehat{P}_{t+1}^{T, \mathcal{LRDB}}}(h_{t+1}) 
        \right]
        +
        \mathcal{L}_{t,rect}^{\pi^L,\PTht}(h_{t},a_{t})
        \Biggr)
        \\
        &=
        \widehat Q_{t,rect}^{\pi^L,\widehat{P}_{t}^{T, \mathcal{LRDB}}}(h_{t},a_{t})
        +
        \mathcal{L}_{t,rect}^{\pi^L,\widehat{P}_{t}^{T, \mathcal{LRDB}}}(h_{t},a_{t})
    \end{align}

\begin{lemmarestated}{lem:rudb}
Let $h_t$ be the history at time step $t$. For any action $a_t \in \mathcal{A}_t$, and future action $a_{t+1}$ selected by the greedy policies $\pi^L, \pi^U$, the following inequalities hold:
\begin{align}
    \RLDBa
    \leq
    r(s_t,a_t)
    +
    \inf_{\PTt \in \UTt}
    \left(
    \widehat{\mathbb{E}}_{s_{t+1}|h_t}^{\PTt}
    \!\left[
    \mathcal{LRDB}_{t+1}^{\pi^L}(h_{t+1})
    \right]
    +
    P^{out} \cdot
    r_{min}\cdot
    (H-1-t)
    \right),
\end{align}

\noindent And:
\begin{align}
    \RUDBa
    \geq
    r(s_t,a_t)
    +
    \inf_{\PTt \in \UTt}
    \left(
    \widehat{\mathbb{E}}_{s_{t+1}|h_t}^{\PTt}
    \!\left[
    \mathcal{URDB}_{t+1}^{\pi^U}(h_{t+1})
    \right]
    +
    P^{out} \cdot
    r_{max}\cdot
    (H-1-t)
    \right).
\end{align}
\end{lemmarestated}

\smallskip
\noindent\hfill\rule{0.5\linewidth}{0.4pt}\hfill
\smallskip

\textbf{Proof of Theorem~\ref{thm:optimal-bounds}.}

% ==================================
% Theorem restatement
% ==================================
\begin{theoremrestated}{thm:optimal-bounds}[Optimal robust value bounds]
Let $h_t$ be an expanded history node at time step $t$, and let $\pi^L$ and $\pi^U$ be the 
lower- and upper-bound policies defined in \eqref{eq:thm_2_pi_L},\eqref{eq:thm_2_pi_U}. Under rectangular transition
uncertainty, the optimal robust value is bounded deterministically by:
% Let $h_t$ be the history at time step $t$. The optimal robust value function at time step $t$ over rectangular uncertainty sets \eqref{eq:rect uncertainty set}, is bounded deterministically by:
% 
\begin{align}
    \RLDB
    \leq
     V_{t,rect}^{*,\mathrm{rob}}(h_t)
    \ \leq \
    \RUDB.
\end{align}

\noindent Subject to the policies: $\pi^U=\{\pi_k^U\}_{k=H-1}^{t}$ and $\pi^L=\{\pi_k^L\}_{k=H-1}^{t}$:
\begin{align}\label{eq:thm_2_pi_L}
    \pi_t^L(h_t) 
    =
    \argmaxover{a_t \in \mathcal{A}_t}
    \Bigl\{
    \RLDBa
    \Bigr\}
    =
    \argmaxover{a_t \in \mathcal{A}_t}
    \Bigl\{
    \widehat Q_{t,rect}^{\pi^L,\widehat{P}_{t}^{T, \mathcal{LRDB}}}(h_{t},a_{t})
    +
    \mathcal{L}_{t,rect}^{\pi^L,\widehat{P}_{t}^{T, \mathcal{LRDB}}}(h_{t},a_{t})
    \Bigr\},
\end{align}
\begin{align} \label{eq:thm_2_pi_U}
    \pi_t^U(h_t) 
    =
    \argmaxover{a_t \in \mathcal{A}_t}
    \Bigl\{
    \RUDBa
    \Bigr\}
    =
    \argmaxover{a_t \in \mathcal{A}_t}
    \Bigl\{
    \widehat Q_{t,rect}^{\pi^U,\widehat{P}_{t}^{T, \mathcal{URDB}}}(h_{t},a_{t})
    +
    \mathcal{U}_{t,rect}^{\pi^U,\widehat{P}_{t}^{T, \mathcal{URDB}}}(h_{t},a_{t})
    \Bigr\},
\end{align}

\end{theoremrestated}

\begin{proof}  
\noindent \\ We prove by induction, that by applying the Bellman optimality operator on upper/lower bounds to the robust value function, we get an upper/lower bounds on the \textbf{optimal} value function.
% ==================================
% Base case
% ==================================
\bigskip \noindent \\ \textbf{Base case} ($t=H-1$): 

\noindent the optimal action-value at the terminal node is:
\begin{align}
    Q_{H-1,rect}^{\mathrm{*,rob}}(h_{H-1},a_{H-1})
    =
    r(s_{H-1},a_{H-1}).
\end{align}

\noindent At the terminal time step ${H-1}$ of the MDP, there is no future transition term,  therefore:
\begin{align}
    \mathcal{U}_{H-1,rect}^{\pi^U} 
    = 
    \mathcal{L}_{H-1,rect}^{\pi^L} 
    = 
    0
\end{align}

\noindent We set the terminal action-value to be the immediate reward:
\begin{align} 
    &\widehat{Q}_{{H-1},rect}^{\mathrm{\pi^L,rob}}(h_{H-1},a_{H-1})
    =
    \widehat{Q}_{{H-1},rect}^{\mathrm{\pi^U,rob}}(h_{H-1},a_{H-1})
    =
    r(s_{H-1},a_{H-1}).
    \\[0.5em]
    &\mathcal{LRDB}_{H-1}^{\pi^L}(h_{H-1},a_{H-1})
    =
    \mathcal{URDB}_{H-1}^{\pi^U}(h_{H-1},a_{H-1})
    =
     r(s_{H-1},a_{H-1}).\label{eq:thm_2_base_case}
\end{align}

\noindent and the base-case for the induction is the best immediate reward:
\begin{align}
    \max_{\substack{a_{H-1} \in \mathcal{A}_{H-1}}}
    \mathcal{LRDB}_{H-1}^{\pi^L}(h_{H-1},a_{H-1})
    \le
    &V_{H-1,rect}^{\mathrm{*,rob}}(h_{H-1})
    \le
    \max_{\substack{a_{H-1} \in \mathcal{A}_{H-1}}}
    \mathcal{URDB}_{H-1}^{\pi^U}(h_{H-1},a_{H-1})
    \\[1em]
    V_{H-1,rect}^{\mathrm{*,rob}}(h_{H-1})
    &\equiv
    \max_{\substack{a_{H-1} \in \mathcal{A}_{H-1}}}
    r(s_{H-1},a_{H-1}).
\end{align}

% ==================================
% Induction hypothesis
% ==================================
\smallskip \noindent \textbf{Induction Hypothesis}: at time step $t+1$, for every $h_{t+1}=(h_t,a_t,s_{t+1})$ with $s_{t+1} \in \Sin(h_{t+1}^-)$ we assume the following holds:
\begin{align}
    \mathcal{LRDB}_{t+1}^{\pi^L}(h_{t+1})
    \leq
     V_{t+1,rect}^{*,\mathrm{rob}}(h_{t+1})
    \ \leq \
    \mathcal{URDB}_{t+1}^{\pi^U}(h_{t+1}).
\end{align}

\noindent This means that at the next depth of the tree $\mathcal{LRDB}_{t+1}^{\pi^L}, \ \mathcal{URDB}_{t+1}^{\pi^U}$ are already  valid lower/upper bounds on the optimal rectangular robust value.

% ==================================
% Induction step
% ==================================
\smallskip \noindent \textbf{Induction Step}: We will now show that the hypothesis holds for time step $t$,  specifically: 
\begin{align}
    \RLDBa
    \le
    Q_{t, rect}^{*,\mathrm{rob}}(h_{t},a_t)
    \le
    \RUDBa
\end{align}

\noindent Writing the the action value in the recursive Bellman form:
\begin{align}
    &Q_{t,\mathrm{rect}}^{*,\mathrm{rob}}(h_t,a_t)
    =
    r(s_t,a_t)
    +
    \inf_{\PTt \in \UTt}
    \mathbb{E}_{s_{t+1}|h_t}^{\PTt}
    \!\left[
    V_{t+1,\mathrm{rect}}^{*,\mathrm{rob}}(h_{t+1})
    \right]
    =
    \\
    &r(s_t,a_t)
    +
    \inf_{\PTt \in \UTt}
    \left(
    \sum_{s_{t+1} \in \Sin}
    \PTt(s_{t+1}|s_t,a_t)
    V_{t+1,\mathrm{rect}}^{*,\mathrm{rob}}(h_{t+1})
    +
    \sum_{s_{t+1} \notin \Sin}
    \PTt(s_{t+1}|s_t,a_t)
    V_{t+1,\mathrm{rect}}^{*,\mathrm{rob}}(h_{t+1})
    \right)
\end{align}

\noindent For simplicity we use the following abbreviation:
% \[
% P^{in}=\sum_{s_{t+1} \in \Sin}
% \PTt(s_{t+1}|s_t,a_t)
% \]
\[
P^{out}=\sum_{s_{t+1} \notin \Sin}
\PTt(s_{t+1}|s_t,a_t)
\]
\noindent by the induction hypothesis at time step $t+1$, and bound the out of support term by $r_{max}/r_{min}\cdot
    (H-1-t)$: 
\begin{align}
    &r(s_t,a_t)
    +
    \inf_{\PTt \in \UTt}
    \left(
    % \sum_{s_{t+1} \in \Sin}
    % \PTt(s_{t+1}|s_t,a_t)\cdot
    \widehat{\mathbb{E}}_{s_{t+1}|h_t}^{\PTt}
    \!\left[
    \mathcal{LRDB}_{t+1}^{\pi^L}(h_{t+1})
    \right]
    +
    % \sum_{s_{t+1} \notin \Sin}
    % \PTt(s_{t+1}|s_t,a_t)
    P^{out} \cdot
    r_{min}\cdot
    (H-1-t)
    \right)
    \leq
    \\[0.5em]
    &Q_{t,\mathrm{rect}}^{*,\mathrm{rob}}(h_t,a_t)
    \leq
    \nonumber \\[0.5em] \nonumber
    &r(s_t,a_t)
    +
    \inf_{\PTt \in \UTt}
    \left(
    % \sum_{s_{t+1} \in \Sin}
    % \PTt(s_{t+1}|s_t,a_t)\cdot
    \widehat{\mathbb{E}}_{s_{t+1}|h_t}^{\PTt}
    \!\left[
    \mathcal{URDB}_{t+1}^{\pi^U}(h_{t+1})
    \right]
    +
    P^{out} \cdot
    r_{max}\cdot
    (H-1-t)
    \right),
\end{align}

\noindent following Lemma \ref{lem:rudb}:
\begin{align}
    \RUDBa
    \geq
    r(s_t,a_t)
    +
    \inf_{\PTt \in \UTt}
    \left(
    \widehat{\mathbb{E}}_{s_{t+1}|h_t}^{\PTt}
    \!\left[
    \mathcal{URDB}_{t+1}^{\pi^U}(h_{t+1})
    \right]
    +
    P^{out} \cdot
    r_{max}\cdot
    (H-1-t)
    \right),
    \\[1em]
    \RLDBa
    \leq
    r(s_t,a_t)
    +
    \inf_{\PTt \in \UTt}
    \left(
    \widehat{\mathbb{E}}_{s_{t+1}|h_t}^{\PTt}
    \!\left[
    \mathcal{LRDB}_{t+1}^{\pi^L}(h_{t+1})
    \right]
    +
    P^{out} \cdot
    r_{min}\cdot
    (H-1-t)
    \right),
\end{align}

\noindent taking maximum over the actions gives:
\begin{align}
    V_{t,rect}^{*,rob}(h_t)
    =
    \max_{\substack{a_t \in \mathcal{A}_t}}
    Q_{t,rect}^{*,rob}(h_t,a_t)
    \le
    \max_{\substack{a_t \in \mathcal{A}_t}}
    \RUDBa
    \triangleq
    \RUDB,
    \\[1em]
    V_{t,rect}^{*,rob}(h_t)
    =
    \max_{\substack{a_t \in \mathcal{A}_t}}
    Q_{t,rect}^{*,rob}(h_t,a_t)
    \geq
    \max_{\substack{a_t \in \mathcal{A}_t}}
    \RLDBa
    \triangleq
    \RLDB,
\end{align}

\noindent which proves the induction step, and the theorem. 

\end{proof}

\smallskip
\noindent\rule{\linewidth}{0.4pt}
\smallskip

%% file: Appendix/Lemma_5_RUDB.tex
\subsection{Proof of Lemma \ref{lem:rudb}}

\textbf{Restating essential definition for the lemma:}

\noindent \\ \textbf{1)} $\mathcal{URDB}/\mathcal{LRDB}$ transition model:
\begin{align}
    \widehat{P}_{t}^{T, \mathcal{URDB}}(\cdot|s_t,a_t)
    =
    \arginfover{\PTht \in \UTht}
    \left \{
    \widehat{\E}_{s_{t+1}|h_t}^{\PTht}
    \Bigl [
    \widehat V_{t+1,rect}^{\pi^U,\widehat{P}_{t+1}^{T, \mathcal{URDB}}}(h_{t+1})
    \Bigr ]
    +
    \mathcal{U}_{t,rect}^{\pi^U, \PTht}(h_{t},a_{t})
    \right \}, \label{eq:lemma_5_PT_U}
    \\[1em]
    \widehat{P}_{t}^{T, \mathcal{LRDB}}(\cdot|s_t,a_t)
    =
    \arginfover{\PTht \in \UTht}
    \left \{
    \widehat{\E}_{s_{t+1}|h_t}^{\PTht}
    \Bigl [
    \widehat V_{t+1,rect}^{\pi^L,\widehat{P}_{t+1}^{T, \mathcal{LRDB}}}(h_{t+1})
    \Bigr ]
    +
    \mathcal{L}_{t,rect}^{\pi^L, \PTht}(h_{t},a_{t})
    \right \},
\end{align}

\noindent \\ \textbf{2)} For simplicity, we denote the probabily mass for the off-support by:
\[
P^{out}=\sum_{s_{t+1} \notin \Sin}
\PTt(s_{t+1}|s_t,a_t)
\]

\small
\noindent\hfill\rule{0.5\linewidth}{0.4pt}\hfill
\smallskip

\textbf{Proof of Theorem~\ref{lem:rudb}.}

% ==================================
% Lemma restated
% ==================================
\begin{lemma}\label{lem:rudb}
Let $h_t$ be the history at time step $t$. For any action $a_t \in \mathcal{A}_t$, and future action $a_{t+1}$ selected by the greedy policies $\pi^L, \pi^U$, the following inequalities hold:
\begin{align}
    \RLDBa
    \leq
    r(s_t,a_t)
    +
    \inf_{\PTt \in \UTt}
    \left(
    \widehat{\mathbb{E}}_{s_{t+1}|h_t}^{\PTt}
    \!\left[
    \mathcal{LRDB}_{t+1}^{\pi^L}(h_{t+1})
    \right]
    +
    P^{out} \cdot
    r_{min}\cdot
    (H-1-t)
    \right),
\end{align}

\noindent And:
\begin{align}
    \RUDBa
    \geq
    r(s_t,a_t)
    +
    \inf_{\PTt \in \UTt}
    \left(
    \widehat{\mathbb{E}}_{s_{t+1}|h_t}^{\PTt}
    \!\left[
    \mathcal{URDB}_{t+1}^{\pi^U}(h_{t+1})
    \right]
    +
    P^{out} \cdot
    r_{max}\cdot
    (H-1-t)
    \right).
\end{align}
\end{lemma}

\begin{proof}
\noindent \\We first focus on the \textbf{upper} bound:
\noindent \\By the definition of the $\RUDBa$ \eqref{eq:urdba}, at time step $t+1$:
\begin{align}
    r(s_t,a_t)
    +
    \inf_{\PTt \in \UTt}
    \Bigl(
    \widehat{\mathbb{E}}_{s_{t+1}|h_t}^{\PTt}
    &\!\left[
    \mathcal{URDB}_{t+1}^{\pi^U}(h_{t+1})
    \right]
    +
    P^{out} \cdot
    r_{max}\cdot
    (H-1-t)
    \Bigr)
    =
    \\
     r(s_t,a_t)
    +
    \inf_{\PTt \in \UTt}
    \Bigl(
    \widehat{\mathbb{E}}_{s_{t+1}|h_t}^{\PTt}
    &\!\left[
     \max_{\substack{a_{t+1} \in \mathcal{A}_{t+1}}}
    \Bigl\{
    \mathcal{URDB}_{t+1}^{\pi^U}(h_{t+1},a_{t+1})
    \Bigr\}
    \right]
    +
    P^{out} \cdot
    r_{max}\cdot
    (H\!-\!1\!-\!t)    
    \Bigr)
    =
    \\
    r(s_t,a_t)
    +
    \inf_{\PTt \in \UTt}
    \Bigl(
    \widehat{\mathbb{E}}_{s_{t+1}|h_t}^{\PTt}
    &\!\left[
    \max_{\substack{a_{t+1} \in \mathcal{A}_{t+1}}}
    \Bigl\{
    \widehat Q_{t+1,rect}^{\pi^U,\widehat{P}_{t+1}^{T, \mathcal{URDB}}}(h_{t+1},a_{t+1})
    +
    \mathcal{U}_{t+1,rect}^{\pi^U, \widehat{P}_{t+1}^{T, \mathcal{URDB}}}(h_{t+1},a_{t+1})
    \Bigr\}
    \right]
    \nonumber \\ 
    +
    P^{out} \cdot
    r_{max}\cdot
    (H-1-t)
    \Bigr)
\end{align}

\noindent The action $a_{t+1}$ that maximizes the expression is the same action that comes from the policy $\pi^U$, namely:
\begin{align}
    a_{t+1}^U
    =
    \pi^U_{t+1}(h_{t+1})
    = 
    \argmaxover{a_{t+1} \in \mathcal{A}_t} 
    \Bigl\{
    \widehat Q_{t+1,rect}^{\pi^U,\widehat{P}_{t+1}^{T, \mathcal{URDB}}}(h_{t+1},a_{t+1})
    +
    \mathcal{U}_{t+1,rect}^{\pi^U,\widehat{P}_{t+1}^{T, \mathcal{URDB}}}(h_{t+1},a_{t+1})
    \Bigr\}
\end{align}

\noindent Therefore we replace $a_{t+1}$ with ${\pi_{t+1}^U}$ $\left( \widehat Q(h_{t+1},\pi^U_{t+1}) \rightarrow \widehat{V}(h_{t+1})\right)$:
% 

% \noindent Therefore we replace $a_{t+1}$ with ${\pi_{t+1}^U}$ $\left( \widehat Q(h_{t+1},\pi^U_{t+1}) \rightarrow \widehat{V}(h_{t+1})\right)$, and use the linearity of the expectation:
% % 
% \begin{align} \label{eq:lemma_5.1}
%     r(s_t,a_t)
%     +
%     \inf_{\PTt \in \UTt}
%     \left(
%     \mathbb{E}_{s_{t+1}|h_t}^{\PTt}
%     \!\left[
%     \widehat V_{t+1,rect}^{\pi^U,\widehat{P}_{t+1}^{T, \mathcal{URDB}}}(h_{t+1})
%     \right]
%     +
%     \mathbb{E}_{s_{t+1}|h_t}^{\PTt}
%     \!\left[
%     \mathcal{U}_{t+1,rect}^{\pi^U, \widehat{P}_{t+1}^{T,\mathcal{URDB}}}(h_{t+1},\pi_{t+1}^U)
%     \right]
%     \right).
% \end{align}

\begin{align} \label{eq:lemma_5.1_main}
    r(s_t,a_t)
    +
    \inf_{\PTt \in \UTt}
    \left(
    \widehat{\mathbb{E}}_{s_{t+1}|h_t}^{\PTt}
    \!\left[
    \widehat V_{t+1,rect}^{\pi^U,\widehat{P}_{t+1}^{T, \mathcal{URDB}}}(h_{t+1})
    +
    \mathcal{U}_{t+1,rect}^{\pi^U, \widehat{P}_{t+1}^{T,\mathcal{URDB}}}(h_{t+1},\pi_{t+1}^U)
    \right]
    \right)
    \nonumber\\
    +
    P^{out} \cdot
    r_{max}\cdot
    (H-1-t)
    \Bigr)
    =
\end{align}

\noindent we split the sum to in/out sampled support:
\begin{align}
    r(s_t,a_t)
    &+
    \inf_{\PTt \in \UTt} \Biggl(
    \sum_{s_{t+1} \in \Sin}
    \PTt(s_{t+1}|s_t,a_t)
    \!\left[
    \widehat V_{t+1,rect}^{\pi^U,\widehat{P}_{t+1}^{T, \mathcal{URDB}}}(h_{t+1})
    +
    \mathcal{U}_{t+1,rect}^{\pi^U, \widehat{P}_{t+1}^{T,\mathcal{URDB}}}(h_{t+1},\pi_{t+1}^U)
    \right]
    \nonumber \\
    &+
    \sum_{s_{t+1} \notin \Sin}
    \PTt(s_{t+1}|s_t,a_t)
    \Bigl[
    r_{max}\cdot
    (H-1-t)
    \Bigr]
    \Biggr)
    =
    \\[1em]
    r(s_t,a_t)
    &+
    \inf_{\PTt \in \UTt} \Biggl(
    \sum_{s_{t+1} \in \Sin}
    \PTt(s_{t+1}|s_t,a_t)
    \!\left[
    \widehat V_{t+1,rect}^{\pi^U,\widehat{P}_{t+1}^{T, \mathcal{URDB}}}(h_{t+1})
    +
    r_{max} \cdot
    \left(H-2-t - 
    \delta_{t+1}^{\pi^U, \widehat{P}_{t+1}^{T,\mathcal{URDB}}}\!\!\!\!\!\!\!\!(s_{t+1}|h_{t+1})
    \right)
    \right]
    \nonumber \\ 
    &+
    \Biggl(1-
    \sum_{s_{t+1} \in \Sin}
    \PTt(s_{t+1}|s_t,a_t) \Biggr)\cdot
    \Bigl[
    r_{max}\cdot
    (H-1-t)
    \Bigr]
    \Biggr)
    =
    \\[1em]
    r(s_t,a_t)
    &+
    \inf_{\PTt \in \UTt} \Biggl(
    \sum_{s_{t+1} \in \Sin}
    \PTt(s_{t+1}|s_t,a_t) \cdot
    \widehat V_{t+1,rect}^{\pi^U,\widehat{P}_{t+1}^{T, \mathcal{URDB}}}(h_{t+1})
    \nonumber \\
    &+
    r_{max}
    \Biggl(
    H-1-t
    -\!\!\!
    \sum_{s_{t+1} \in \Sin}\!\!
    \PTt(s_{t+1}|s_t,a_t) \cdot
    \left[
    1+
    \delta_{t+1}^{\pi^U, \widehat{P}_{t+1}^{T,\mathcal{URDB}}}\!\!\!\!\!\!\!\!\!(s_{t+1}|h_{t+1})
    \right]
    \!\Biggr)
    \Biggr)\label{eq:lemma_5.2_deltas} 
    \leq
\end{align}

\noindent as defined in \eqref{eq:thm_1_delta}, $\delta$ uses infimum/suprimum based on the sign of $r_{max}$ :
\begin{align} 
\begin{cases}
    \delta_{t+1}^{\pi^U} \le \delta_{t+1}^{\pi^U, \widehat{P}_{t+1}^{T,\mathcal{URDB}}}, & r_{max}\geq 0, \\[1em]
    \delta_{t+1}^{\pi^U} \geq \delta_{t+1}^{\pi^U, \widehat{P}_{t+1}^{T,\mathcal{URDB}}}, & r_{max}< 0.
    \end{cases} 
\end{align}

\noindent Therefore:
\begin{align}
    r(s_t,a_t)
    &+
    \inf_{\PTt \in \UTt} \Biggl(
    \sum_{s_{t+1} \in \Sin}
    \PTt(s_{t+1}|s_t,a_t) \cdot
    \widehat V_{t+1,rect}^{\pi^U,\widehat{P}_{t+1}^{T, \mathcal{URDB}}}(h_{t+1})
    \nonumber \\
    &+
    r_{max}
    \Bigl(
    H-1-t
    -\!\!\!
    \sum_{s_{t+1} \in \Sin}\!\!
    \PTt(s_{t+1}|s_t,a_t) \cdot
    \left[
    1+
    \delta_{t+1}^{\pi^U}(s_{t+1}|h_{t+1})
    \right]
    \!\Bigr)
    \Biggr)
    =
    \\[1em]
    r(s_t,a_t)
    &+
    \inf_{\PTt \in \UTt} \Biggl(
    \sum_{s_{t+1} \in \Sin}
    \PTt(s_{t+1}|s_t,a_t)\cdot
    \widehat V_{t+1,rect}^{\pi^U,\widehat{P}_{t+1}^{T, \mathcal{URDB}}}(h_{t+1})
    +
    \mathcal{U}_{t,rect}^{\pi^U, \PTt}(h_{t},a_{t})
    \Biggr)
    =
\end{align}

\noindent \\ Using Lemma \ref{lem:Rmap} together with the restriction map $\RTt$ from \ref{RTmap}, that project the full uncertainty set onto the projected uncertainty set, we get the following equalities:
\begin{align}
    &r(s_t,a_t)
    +
    \inf_{\PTt \in \UTt} \left(
    \widehat{\mathbb{E}}_{s_{t+1}|h_t}^{\RT\left(\PTt\right)}
    \left[
    \widehat V_{t+1,rect}^{\pi^U,\widehat{P}_{t+1}^{T, \mathcal{URDB}}}(h_{t+1})
    \right]
    +
    \mathcal{U}_{t,rect}^{\pi^U, \RT\left(\PTt\right)}(h_{t},a_{t})
    \right)
    =
    \\[1em]
    &r(s_t,a_t)
    +
    \inf_{\PTht \in \UTht} \left(
    \widehat{\mathbb{E}}_{s_{t+1}|h_t}^{\PTht}
    \left[
    \widehat V_{t+1,rect}^{\pi^U,\widehat{P}_{t+1}^{T, \mathcal{URDB}}}(h_{t+1})
    \right]
    +
    \mathcal{U}_{t,rect}^{\pi^U, \PTht}(h_{t},a_{t})
    \right)
    \triangleq
    \\[1em]
    &\RUDBa,
\end{align}

\noindent which proves the upper bound.

\noindent The \textbf{lower} bound follows the same proof structure, with the following modifications:\\[1em]
\begin{enumerate}
    \item Replace $\mathcal{URDB}$ with $\mathcal{LRDB}$.
    \item Replace $\mathcal{U}_{t+1}$ with $\mathcal{L}_{t+1}$.
    \item Replace $\pi^U$ with $\pi^L$.
    \item Replace $r_{max}$ with $r_{min}$.
    \item Change the inequality directions.
    \item Analogue justification to the transition after \eqref{eq:lemma_5.2_deltas}:
    \begin{align} 
    \begin{cases}
        \delta_{t+1}^{\pi^L} \geq \delta_{t+1}^{\pi^L, \widehat{P}_{t+1}^{T,\mathcal{LRDB}}}, & r_{min}\geq 0, \\[1em]
        \delta_{t+1}^{\pi^L} \leq \delta_{t+1}^{\pi^L, \widehat{P}_{t+1}^{T,\mathcal{LRDB}}}, & r_{min}< 0.
        \end{cases} 
    \end{align}
\end{enumerate}

\end{proof}

\smallskip
\noindent\rule{\linewidth}{0.4pt}
\smallskip

%% file: Appendix/corollary_2_convergance_optimal.tex
\subsection{Proof of Corollary \ref{cor:optimal-convergence} (Asymptotic optimality)}\label{cor:cor_2_proof} 

\begin{corollaryrestated}{cor:optimal-convergence}[Asymptotic optimality]
As the sampled state support expands to the full state space at every node,
$
\Sin(h_k^-)\uparrow\Sfull,
\ k=t,\ldots,H,
$
both deterministic bounds converge to the optimal rectangular robust value:
\begin{align}
\mathcal{LRDB}_{t}^{\pi^L}
\left(h_t,\pi_t^L(h_t)\right)
&\longrightarrow
V_{t,\mathrm{rect}}^{*,\mathrm{rob}}(h_t),
\\
\mathcal{URDB}_{t}^{\pi^U}
\left(h_t,\pi_t^U(h_t)\right)
&\longrightarrow
V_{t,\mathrm{rect}}^{*,\mathrm{rob}}(h_t).
\end{align}
\end{corollaryrestated}

% ==================================
% Proof
% ==================================

\begin{proof}
\noindent \\ We define the following expressions, essential for the proof:
\begin{align} \label{eq:cor_2_epsilon}
    \epsilon
    \triangleq
    \max_{k\in[t,H-1]}
    \sup_{\substack{\pi_k,a_k,h_k,\\\PThk \in \UTh[k]}}
    \max \left\{
   \left| \mathcal{U}_{k,rect}^{\pi,\PThk}(h_{k},a_{k}) \right |,
   \left| \mathcal{L}_{k,rect}^{\pi,\PThk}(h_{k},a_{k}) \right |
    \right\}
\end{align}

% \noindent For every $k\in[t,H-1]$:
% 
\begin{align}\label{eq:cor_2_G}
    G_k
    \triangleq
    \sup_{h_k}
    \left|
    \mathcal{URDB}_{k}^{\pi^U}(h_k)
    -
    \mathcal{LRDB}_{k}^{\pi^L}(h_k)
    \right|,
    \quad
    k\in[t,H-1]
\end{align}

% \noindent and for $k=H-1$:
% % 
% \begin{align}
%     G_{H-1}
%     =
%     0
% \end{align}

\noindent We will show the following for the proof:
\begin{enumerate}
    \item $G_t \leq 4\cdot(H-1-t)\cdot\epsilon$
    \item $\epsilon \xrightarrow{\substack{\Sin \rightarrow \Sfull}} 0$
    \item $ V_{t,rect}^{*,\mathrm{rob}}(h_t) \in [\RLDB,\RUDB]$, and by claim 1 and 2 $G_t \xrightarrow{\substack{\Sin \rightarrow \Sfull}} 0$, which proves the corollary.
\end{enumerate}

\noindent \textbf{Step 1:}
\noindent \\ base case $ G_{H-1}=0$ by equations \ref{eq:thm_2_base_case}.

\noindent For some  $k\in[t,H-2]$:
\begin{align}
    \left|
    \mathcal{URDB}_{k}^{\pi^U}(h_k)
    -
    \mathcal{LRDB}_{k}^{\pi^L}(h_k)
    \right|
    \leq
    \max_{a_k}
    \left|
    \mathcal{URDB}_{k}^{\pi^U}(h_k,a_k)
    -
    \mathcal{LRDB}_{k}^{\pi^L}(h_k,a_k)
    \right|
\end{align}

\noindent we continue with arbitrary $a_k$ for now:
\begin{align}
    \max_{a_k}
    \Biggl|
    \inf_{\PThk \in \UTh[k]}
    \Biggl(
    \widehat{\mathbb{E}}_{s_{k}|h_k}^{\PThk}
    \!\left[
    \widehat V_{k+1,rect}^{\pi^U,\widehat{P}_{k+1}^{T, \mathcal{URDB}}}(h_{k+1}) 
    \right]
    +
    \mathcal{U}_{k,rect}^{\pi^U,\PThk}(h_{k},a_{k})
    \Biggr)
    \nonumber \\
    -
    \inf_{\PThk \in \UTh[k]}
    \Biggl(
    \widehat{\mathbb{E}}_{s_{k}|h_k}^{\PThk}
    \!\left[
    \widehat V_{k+1,rect}^{\pi^L,\widehat{P}_{k+1}^{T, \mathcal{LRDB}}}(h_{k+1}) 
    \right]
    +
    \mathcal{L}_{k,rect}^{\pi^L,\PThk}(h_{k},a_{k})
    \Biggr)
    \Biggr|
    \leq
\end{align}

\begin{align}
    \max_{a_k}
    \sup_{\PThk \in \UTh[k]}
     \Biggl|
    \Biggl(
    \widehat{\mathbb{E}}_{s_{k}|h_k}^{\PThk}
    \!\left[
    \widehat V_{k+1,rect}^{\pi^U,\widehat{P}_{k+1}^{T, \mathcal{URDB}}}(h_{k+1}) 
    \right]
    +
    \mathcal{U}_{k,rect}^{\pi^U,\PThk}(h_{k},a_{k})
    \nonumber \\
    -
    \widehat{\mathbb{E}}_{s_{k}|h_k}^{\PThk}
    \!\left[
    \widehat V_{k+1,rect}^{\pi^L,\widehat{P}_{k+1}^{T, \mathcal{LRDB}}}(h_{k+1}) 
    \right]
    +
    \mathcal{L}_{k,rect}^{\pi^L,\PThk}(h_{k},a_{k})
    \Biggr)
    \Biggr|
    \leq
\end{align}

\noindent using linearity of expectation, and definition of \eqref{eq:cor_2_epsilon}:

\begin{align}
    \sup_{\PThk \in \UTh[k]}
     \Biggl|
    \Biggl(
    \widehat{\mathbb{E}}_{s_{k}|h_k}^{\PThk}
    \!\left[
    \widehat V_{k+1,rect}^{\pi^U,\widehat{P}_{k+1}^{T, \mathcal{URDB}}}(h_{k+1}) 
    -
    \widehat V_{k+1,rect}^{\pi^L,\widehat{P}_{k+1}^{T, \mathcal{LRDB}}}(h_{k+1}) 
    \right]
    \Biggr)
    \Biggr|
    +
    2\epsilon
    \leq
\end{align}

\noindent Using equations \ref{eq:urdba}, \ref{eq:lrdba}, and we also get rid of the $sup$:
\begin{align}
    % \sup_{\PThk \in \UTh[k]}
     \Biggl|
    \Bigl(
    \mathcal{URDB}_{k+1}^{\pi^U}(h_{k+1})
    -
    \mathcal{U}_{k+1,rect}^{\pi^U, \widehat{P}_{k+1}^{T, \mathcal{URDB}}}(h_{k+1}, \pi_{k+1})
    \bigr)
    -
    \Bigl(
    \mathcal{LRDB}_{k+1}^{\pi^L}(h_{k+1})
    \nonumber \\
    -
    \mathcal{L}_{k+1,rect}^{\pi^L, \widehat{P}_{k+1}^{T, \mathcal{LRDB}}}(h_{k+1}, \pi_{k+1})
    \Bigr)
    \Biggr|
    +
    2\epsilon
    \leq
\end{align}

\begin{align}
    % \sup_{\PThk \in \UTh[k]}
    &\left|
    \mathcal{URDB}_{k+1}^{\pi^U}(h_{k+1})
    -
    \mathcal{LRDB}_{k+1}^{\pi^L}(h_{k+1})
    \right|
    +
    4\epsilon
    \leq
    \\[0.5em]
    \sup_{h_{k+1}}
    &\left|
    \mathcal{URDB}_{k+1}^{\pi^U}(h_{k+1})
    -
    \mathcal{LRDB}_{k+1}^{\pi^L}(h_{k+1})
    \right|
    +
    4\epsilon
    =
    \\[0.5em]
    &G_{k+1}
    +
    4\epsilon.
\end{align}

\noindent Unrolling $H-1$ times from $G_{H-1}$ gives:
\begin{align}
    G_t
    \leq
    4\cdot(H-1-t)\cdot\epsilon
\end{align}

\noindent \textbf{Step 2:}

\noindent Restating the rectangular upper and lower bounds \eqref{eq:thm_1_rect_bounds} in theorem \ref{thm:bounds}, with the fixed model version:
\begin{align}
    &\mathcal{U}_{k,rect}^{\pi, \PThk}(h_k)
    \triangleq
    r_{max}
    \Bigl (
    H -1- k
    -
    \delta_{k}^{\pi, \PThk}(s_k|h_k)
    \Bigr )
    \\[1em]
    &\mathcal{L}_{k,rect}^{\pi, \PThk}(h_k)
    \triangleq
    r_{min}
    \Bigl (
    H -1- k
    -
    \delta_{k}^{\pi, \PThk}(s_k|h_k)
    \Bigr )
    \\[1em]
    &\delta_{k}^{\pi, \PThk}(s_k|h_k) 
    =
    % \inf_{\PThk \in \UTh[k]}
    % \left(
    \sum _{s_{k+1} \in \Sin(h_{k+1}^-)}
    \PThk\bigl(s_{k+1}\mid s_{k},\pi_{k}(h_k)\bigr)\,
    \left [
    1 
    +
    \delta_{k+1}^{\pi}(s_{k+1})
    \right ]
    % \right)
    \\[1em]
    &\delta_{H-1}^{\pi, \PThk}(s_{H-1})
    \equiv
    0
\end{align}

\noindent As the sampled state support increases:
\begin{align}
    \sum_{s_{k:H-1} \in \Sin(h_{H-1}^-)}
    \prod_{i=k}^{H-1}
    \PTk[i]\bigl(s_{i+1}\mid s_i,\pi_i(h_i)\bigr)
    \xrightarrow{
    \substack{
    \Sin \rightarrow \Sfull
    }}
    1,
\end{align}

\noindent therefore:
\begin{align}
    \delta_{k}^{\pi, \PThk}(s_k|h_k)
    \xrightarrow{
    \substack{
    \Sin \rightarrow \Sfull
    }}
    H - 1- k,
\end{align}

\noindent which induces:
\begin{align}
    \mathcal{U}_{k,rect}^{\pi, \PThk}(h_k)
    \xrightarrow{
    \substack{
    \Sin \rightarrow \Sfull
    }}
    0,
    \\
    \mathcal{L}_{k,rect}^{\pi, \PThk}(h_k)
    \xrightarrow{
    \substack{
    \Sin \rightarrow \Sfull
    }}
    0.
\end{align}

\noindent Therefore:
\[
\epsilon \xrightarrow{\substack{\Sin \rightarrow \Sfull}} 0
\]

\noindent \textbf{Step 3:}

\begin{align}
    0
    \leq
    \RUDB
    -
    V_{t,rect}^{*,\mathrm{rob}}(h_t)
    \leq
    G_t,
    \\
    0
    \leq
    V_{t,rect}^{*,\mathrm{rob}}(h_t)
    -
    \RLDB
    \leq
    G_t,
\end{align}

\noindent we have shown that $G_t \xrightarrow{\substack{\Sin \rightarrow \Sfull}} 0$, which then concludes the proof.
\end{proof}

%% file: Appendix/complexity.tex
\section{Robust Backup under TV Uncertainty}
\label{app:tv-robust-backup}

For the TV uncertainty sets used in our experiments, the inner optimization
in \eqref{eq:RLDBa}, \eqref{eq:RUDBa}, can be solved directly on the sampled successor support.
Let
\[
\mathring p_{\mathrm{out}}
=
1-\sum_{s'\in\Sin(h,a)}
\mathring P^T(s'\mid s,a)
\]
denote the nominal probability mass outside the currently sampled support.
Since projection restricts the transition row without renormalization, all
unsampled successors can be represented by a single aggregate outside state.
Consequently, the projected uncertainty set is equivalently represented by
a TV ball over the $|\Sin(h,a)|+1$ dimensional vector consisting of the
sampled successor probabilities and the aggregate outside mass. The aggregate outside state represents the omitted probability mass, with its
contribution determined by the corresponding projected backup.

The resulting inner problem is a linear optimization over a TV ball. We solve
it by the standard mass-transfer procedure: probability mass of at most
$\rho$ is shifted from successors with the largest continuation values toward
those with the smallest values, subject to the simplex constraints. Sorting
the continuation values gives a per-backup complexity of
$O(|\Sin(h,a)|\log |\Sin(h,a)|)$ and $O(|\Sin(h,a)|)$ memory. Thus, the
robust backup depends on the currently sampled successor support rather than
on the cardinality of the full state space.

%% file: Appendix/experiments.tex
\section{Experiments} \label{exp:experiments}

% ==================================
% Frozen lake fomain
% ==================================
\subsection{Frozen Lake Domain}
\label{app:frozen-lake}

We evaluate the planners on a family of Frozen Lake domains represented by
two-dimensional grids containing traversable frozen cells, terminal holes, a
start cell, and a goal cell. Each grid cell corresponds to an MDP state, and
we additionally introduce an absorbing terminal state $s_{\mathrm{done}}$.
The action space is
$\mathcal A=\{\textsc{left},\textsc{down},\textsc{right},\textsc{up}\}$,
corresponding to the four cardinal movement directions. Reaching either a
hole or the goal terminates the episode, after which the process transitions
to $s_{\mathrm{done}}$. The objective is to reach the goal while avoiding
holes. Different experiments vary the grid size, obstacle layout, start and
goal locations, and planning horizon; their specific configurations are
summarized in Section~\ref{app:experiment-configurations}.

\begin{figure}[h]
    \centering
    \includegraphics[width=0.3\linewidth]{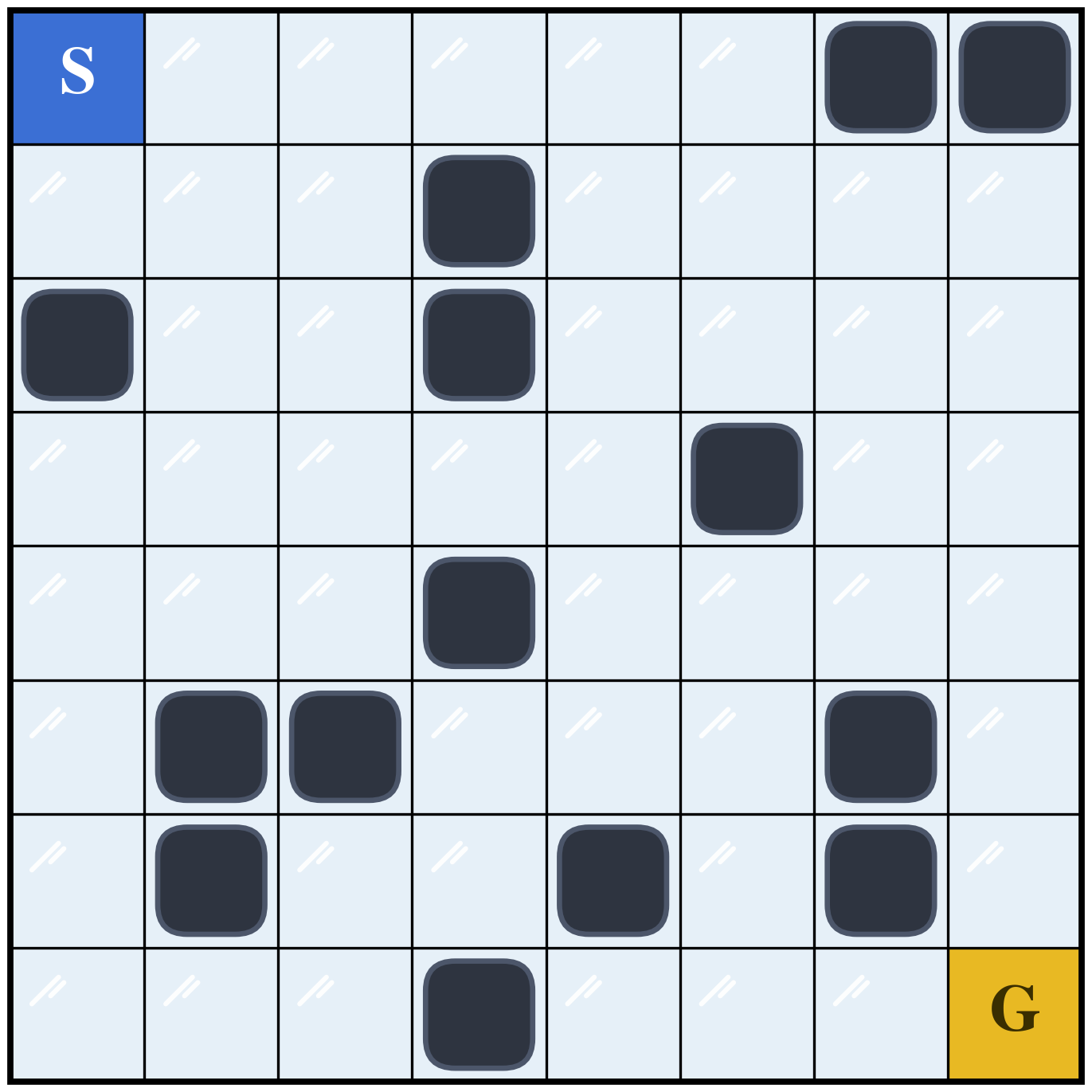}
    \caption{Frozen Lake Map Example}
    \label{fig:Frozen Lake Map Example}
\end{figure}

% ==================================
% Ring Kernel
% ==================================
\subsection{Ring Transition Model}
\label{app:ring-kernel}

The standard Frozen Lake dynamics have a small successor support. To evaluate
planning with substantially larger transition supports, we use a stochastic
\emph{ring transition model}. For a state $s$, let
\[
\mathcal R_d(s)
=
\left\{
s' : \|s'-s\|_\infty=d
\right\},
\qquad d=1,\ldots,k,
\]
denote the Chebyshev ring of radius $d$ around $s$. For an interior state,
$\mathcal R_d(s)$ contains $8d$ cells.

Under action $a$, the intended neighboring cell receives probability
$1-p_{\mathrm{slip}}$. The remaining probability $p_{\mathrm{slip}}$ is
distributed over the $k$ rings. The fraction assigned to ring $d$ is
\[
q_d
=
\frac{\beta^{d-1}}
{\sum_{\ell=1}^{k}\beta^{\ell-1}},
\]
and this mass is distributed uniformly among the $8d$ cells in that ring.
Thus, for an interior state, each cell in ring $d$ receives slip probability
\[
p_{\mathrm{slip}}\frac{q_d}{8d}.
\]
The intended successor is itself contained in the first ring and therefore
receives both its ring contribution and the additional
$1-p_{\mathrm{slip}}$ intended-transition probability.

Near the grid boundary, ring locations falling outside the grid are clipped
to the nearest valid cell, and probability masses mapping to the same cell
are accumulated. Holes block transitions: if a transition crosses an adjacent
hole, the transition is redirected to that hole, which terminates the episode.

\begin{figure}[h]
    \centering
    \includegraphics[width=\linewidth]{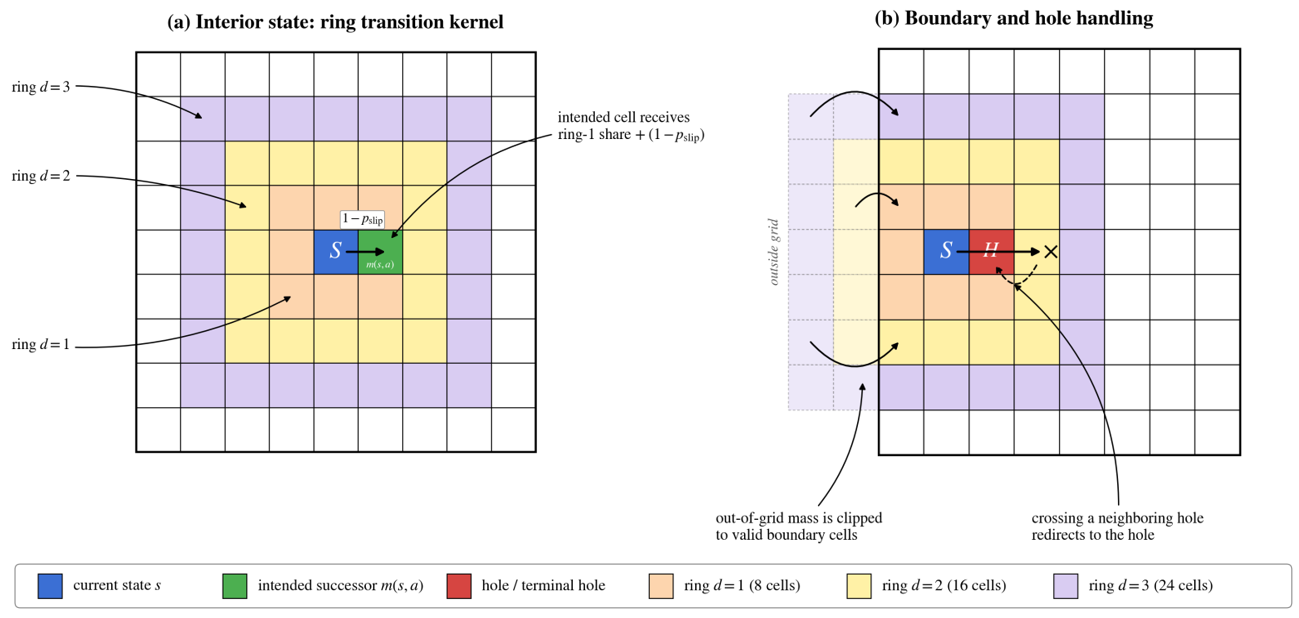}
    \caption{Illustration of the ring transition model used in the Frozen Lake experiments.}
    \label{fig:ring-transition-model}
\end{figure}

% ==================================
% Transition uncertainty model
% ==================================
\subsection{Transition Uncertainty Model}
\label{app:transition-uncertainty}

To account for model uncertainty, we construct a rectangular uncertainty set
independently around each nominal transition row. Specifically, for every
state-action pair $(s,a)$,
\begin{equation}
\mathcal P^T(s,a)
=
\left\{
P^T(\cdot\mid s,a):
\operatorname{TV}
\!\left(
P^T(\cdot\mid s,a),
\mathring{P}^{T}(\cdot\mid s,a)
\right)
\le \rho
\right\},
\label{eq:exp-tv-set}
\end{equation}
where $\rho$ is the uncertainty radius and
\[
\operatorname{TV}(p,q)
=
\frac{1}{2}\|p-q\|_1 .
\]

The uncertainty set preserves the support of the nominal transition model:
states assigned zero probability by
$\mathring{P}^{T}(\cdot\mid s,a)$ cannot receive probability mass under an
admissible model. Thus, robustness accounts for uncertainty in the transition
probabilities over the ring-kernel support without introducing additional
successor states.

In experiments involving model mismatch, the environment is executed using a
slip probability $p_{\mathrm{true}}$, whereas the planner constructs its
nominal model using $p_{\mathrm{slip}}$. The uncertainty radius $\rho$ is
chosen to contain the true model within the corresponding TV uncertainty
set. The exact values used in each experiment are reported in
Section~\ref{app:experiment-configurations}.

% ==================================
% Reward model
% ==================================
% \subsection{Reward Model}
% \label{app:reward-model}

% We use the distance-shaped reward of ~\cite{shazman2026online}.
% For a nonterminal free cell $s$, let $m(s)$ denote its Manhattan distance to
% the goal, computed directly on the grid without treating holes as obstacles.
% The reward is
% \begin{equation}
% r(s)
% =
% w
% \left(
% \frac{1}{1+m(s)}
% \right)^{p},
% \label{eq:frozen-lake-reward}
% \end{equation}
% with $w=0.1$ and $p=3$. Terminal rewards override this shaping: the goal
% receives reward $3$, while holes and the absorbing state
% $s_{\mathrm{done}}$ receive reward $0$. Hence, $r\in[0,3]$.

% The cubic inverse-distance shaping decays rapidly away from the goal, so most
% free cells receive only a small intermediate reward and the task remains
% predominantly driven by reaching the terminal goal.

\subsection{Reward Model}
\label{app:reward-model}

We use two distance-based reward models across the experiments.

\paragraph{Inverse-Manhattan reward.}
For the main planner comparison and Ablations~1--2, we use the
distance-shaped reward of Shazman et al.~\cite{shazman2026online}. For a
nonterminal free cell $s$, let $m(s)$ denote its Manhattan distance to the
goal, computed directly on the grid without treating holes as obstacles:
\begin{equation}
r_{\mathrm{IM}}(s)
=
w
\left(
\frac{1}{1+m(s)}
\right)^p .
\label{eq:frozen-lake-reward}
\end{equation}
We use $w=0.1$ and $p=3$, with goal reward $3$ and rewards $0$ for holes
and $s_{\mathrm{done}}$. The cubic inverse-distance shaping decays rapidly
away from the goal, so most free cells receive only a small intermediate
reward and the task remains predominantly driven by reaching the terminal
goal.

\paragraph{Positive-distance reward.}
For several experiments we instead use the shortest-path distance $d(s)$
from $s$ to the goal, computed while routing around holes. Let
$d_{\max}$ be the largest such distance on the map and define
\[
c(s)=\frac{d_{\max}-d(s)}{d_{\max}}\in[0,1].
\]
The positive-distance reward is
\[
r_{+}(s)=w d_{\max} c(s)^p.
\]
Thus, free-cell rewards increase as the agent approaches the goal, with the
distance computed according to the traversable structure of the map rather
than directly through holes. Goal and hole rewards are specified separately
and override the shaping value, while $r(s_{\mathrm{done}})=0$.

% \paragraph{Negative-distance reward.}
% We also consider the corresponding negative shaping
% \[
% r_{-}(s)
% =
% -w d_{\max}\left(1-c(s)^p\right).
% \]
% Here, states farther from the goal receive increasingly negative reward,
% while states near the goal approach zero before the terminal reward override.
% This preserves the same distance ordering while providing experiments with
% negative rewards, allowing the two-sided certificates to be evaluated when
% $r_{\min}<0$. As above, goal and hole rewards override the shaping value and
% $r(s_{\mathrm{done}})=0$.

When a fixed reward span is used, we set
$w=\mathrm{span}/d_{\max}$, making the free-cell reward range independent of
the map diameter. The specific reward model and parameters used by each
experiment are summarized in Section~\ref{app:experiment-configurations}.

% ==================================
% Planning protocol
% ==================================
\subsection{Planning Protocol}
\label{app:planning-protocol}

All online planning experiments use a receding-horizon protocol. At each
environment step, the planner constructs a finite search tree of depth $H$
rooted at the current state and performs planning using its available
computation budget. The selected root action is then executed in the
environment, the resulting next state is observed, and a new depth-$H$
planning problem is initialized from that state. Thus, $H$ denotes the
planner's lookahead horizon rather than the episode length.

For model-mismatch experiments, planning is performed using the nominal
transition model $\mathring{P}^{T}$ parameterized by $p_{\mathrm{slip}}$,
whereas executed transitions are sampled from the environment model
parameterized by $p_{\mathrm{true}}$. Episodes terminate upon reaching either
the goal or a hole.

% ==================================
% Experiment configuration
% ==================================
% \subsection{Experiment Configurations}
% \label{app:experiment-configurations}

% Table~\ref{tab:experiment-configurations} summarizes the parameters used in
% the main experiments. Unless stated otherwise, all experiments use the domain,
% transition kernel, uncertainty model, reward function, and planning protocol
% described above.

% \begin{table}[h]
% \centering
% \small
% \renewcommand{\arraystretch}{1.2}
% \setlength{\tabcolsep}{4pt}
% \caption{Configuration of the experiments reported in the main paper.}
% \label{tab:experiment-configurations}
% \begin{tabular}{|l|c|c|c|c|c|c|c|c|c|}
% \hline
% \textbf{Experiment}
% & \textbf{Grid}
% & \textbf{$H$}
% & \textbf{$k$}
% & \textbf{$\beta$}
% & \textbf{$p_{\mathrm{slip}}$}
% & \textbf{$p_{\mathrm{true}}$}
% & \textbf{$\rho$}
% & \textbf{Budget}
% & \textbf{Runs} \\
% \hline

% Fig.~1a
% & $8\times8$
% & $6$
% & $1$
% & $1$
% & $0.3$
% & --
% & $0.1$
% & --
% & $50$ \\
% \hline

% Fig.~1b
% & $8\times8$
% & $3$
% & $2$
% & $1$
% & $0.2$
% & --
% & $0.1$
% & varied
% & $20$ \\
% \hline

% Fig.~2
% & $16\times16$
% & $\{2,5,11\}$
% & $8$
% & $1$
% & $0.01$--$0.6$
% & $0.05$
% & $0.04$--$0.54$
% & $20 - 200$
% & $20$ \\
% \hline

% Fig.~3
% & $\!\!(10 \!\times\! 10)\!\! \rightarrow \!\!(10^3 \!\times\! 10^3)\!\!$
% & $5$
% & $\frac{N-1}{2}$
% & $1$
% & $0.3$
% & --
% & $0.1$
% & $80$
% & $5$ \\
% \hline

% \end{tabular}
% \end{table}

\subsection{Experiment Configurations}
\label{app:experiment-configurations}

Table~\ref{tab:experiment-configurations} summarizes the parameters used in
the main experiments. Unless stated otherwise, all experiments use the domain,
transition kernel, uncertainty model, and planning protocol described above.

\begin{table}[h]
\centering
\small
\renewcommand{\arraystretch}{1.2}
\setlength{\tabcolsep}{6pt}
\caption{Configuration of the experiments reported in the main paper.}
\label{tab:experiment-configurations}

\begin{tabular}{|l|c|c|c|c|}
\hline
\textbf{Parameter}
& \textbf{Fig.~1a}
& \textbf{Fig.~1b}
& \textbf{Fig.~2}
& \textbf{Fig.~3} \\
\hline

Grid
& $8\times8$
& $8\times8$
& $16\times16$
& $(10\times10)$--$(10^3\times10^3)$ \\
\hline

$H$
& $6$
& $3$
& $\{2,5,11\}$
& $5$ \\
\hline

$k$
& $1$
& $2$
& $8$
& $\frac{n-1}{2}$ \\
\hline

$\beta$
& $1$
& $1$
& $1$
& $0.39$--$0.99$ \\
\hline

$p_{\mathrm{slip}}$
& $0.3$
& $0.2$
& $0.01$--$0.6$
& $0.3$ \\
\hline

$p_{\mathrm{true}}$
& --
& --
& $0.05$
& -- \\
\hline

$\rho$
& $0.1$
& $0.1$
& $0.04$--$0.54$
& $0.1$ \\
\hline

Budget
& --
& varied
& $20$--$200$
& $80$ \\
\hline

Runs
& $50$
& $20$
& $20$
& $5$ \\
\hline

Reward
& $r_{+}$
& $r_{+}$
& $r_{\mathrm{IM}}$
& $r_{+}$ \\
\hline

\end{tabular}
\end{table}

Figure~1a uses $r_{+}$ with goal reward $1$, hole reward $-1$,
$\mathrm{span}=1$, and $p=0.5$; Figure~1b uses the same specification with
$p=1$. Figure~2 uses $r_{\mathrm{IM}}$ with goal reward $3$, hole reward $0$,
$w=0.1$, and $p=3$. Figure~3 uses $r_{+}$ with goal reward $1$, hole reward
$0$, and $p=1$.

% ==================================
% Ablation 1
% ==================================
\subsubsection{Ablation 1: Robustness to Environment Slipperiness and Model Mismatch}
\label{app:ablation-slipperiness}

We first examine how planner performance changes as the true environment
becomes increasingly slippery, and how this interacts with model mismatch.
We fix the planning depth to $3$ and use the $16\times16$ Frozen Lake domain.
For the main experiment, the ring radius is $k=8$ with $\beta=1$.
We vary the true environment slip probability as
$p_{\mathrm{true}}\in\{0.01,0.1,0.2,0.3\}$, corresponding to the four
columns of Figure~\ref{fig:ablation-slipperiness-k8}. Within each column,
$p_{\mathrm{true}}$ is fixed while the planner-assumed slip probability
$p_{\mathrm{slip}}$ is varied from $0.01$ to $0.6$. Hence, moving across
columns makes the true environment increasingly difficult, whereas moving
along the $x$-axis changes the model assumed by the planner while leaving the
true environment unchanged. For each pair, the TV uncertainty radius is set
to the smallest value that contains the true transition model. Each configuration is evaluated over $500$ independent seeds.

\begin{figure}[h]
    \centering
    \includegraphics[width=\linewidth]{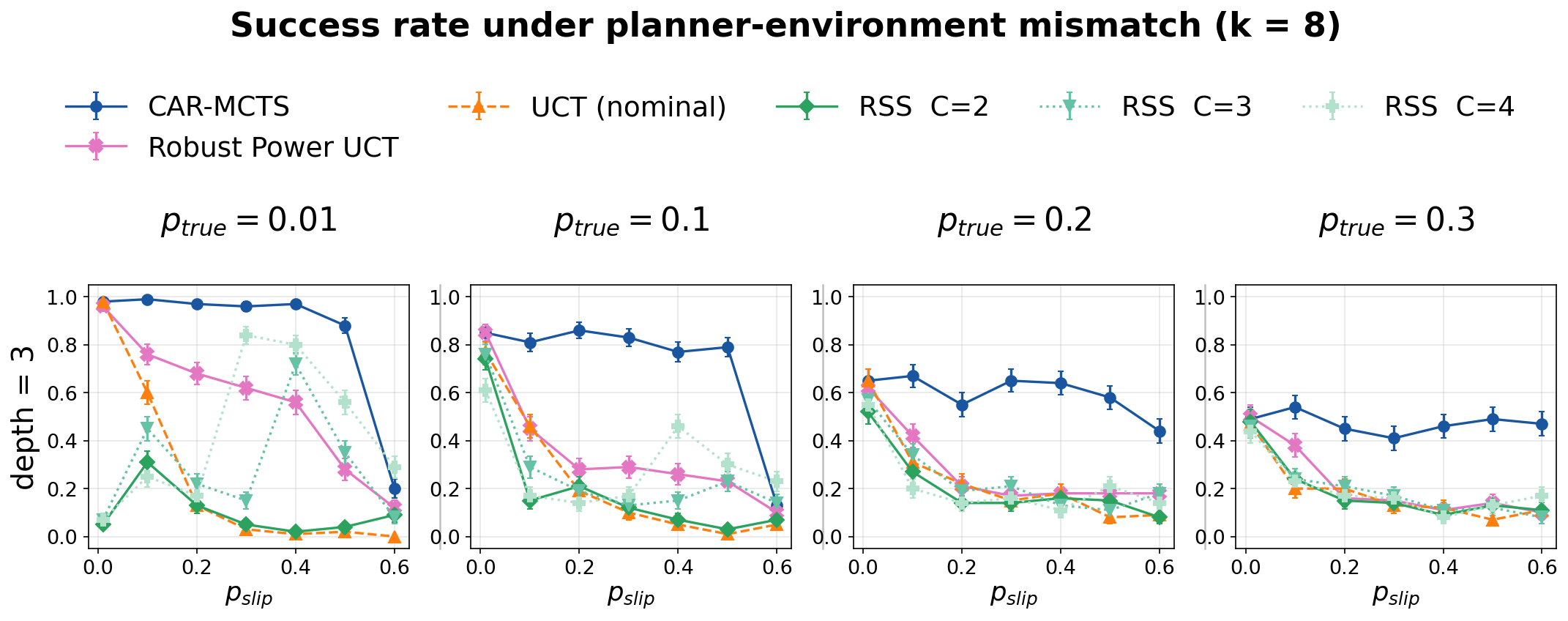}
    \caption{
    Success rate at planning depth $3$ as the true environment becomes more
    slippery and the planner-assumed transition model varies. Columns
    correspond to $p_{\mathrm{true}}\in\{0.01,0.1,0.2,0.3\}$, while the
    $x$-axis shows the planner-assumed $p_{\mathrm{slip}}$. The ring radius is
    $k=8$. $500$ seeds.
    }
    \label{fig:ablation-slipperiness-k8}
\end{figure}

As $p_{\mathrm{true}}$ increases, success rates decrease for all methods,
reflecting the increasing difficulty of controlling the agent when the true
environment becomes more slippery. CAR-MCTS degrades substantially more
gracefully than the baselines and achieves the highest or near-highest success
rate across the tested environment difficulties. Importantly, this behavior
persists as $p_{\mathrm{slip}}$ moves away from $p_{\mathrm{true}}$: within
each column the true environment is unchanged, yet CAR-MCTS remains
comparatively stable over increasing mismatch between the planning and
true environment models.

The RSS variants are not restricted to the same sampling budget as the MCTS
methods. CAR-MCTS, nominal UCT, and Robust Power UCT use $20$ simulations per
decision, corresponding to $60$ transition samples at depth $3$. In contrast,
the recursive sparse-sampling expansion requires substantially more transition
samples as $C$ increases:

\begin{table}[h]
\centering
\small
\renewcommand{\arraystretch}{1.15}
\begin{tabular}{|l|c|c|}
\hline
\textbf{Method} & \textbf{Samples per decision} & \textbf{Relative to MCTS} \\
\hline
MCTS methods & $60$ & $1\times$ \\
\hline
RSS $C=2$ & $584$ & $9.7\times$ \\
\hline
RSS $C=3$ & $1884$ & $31.4\times$ \\
\hline
RSS $C=4$ & $4368$ & $72.8\times$ \\
\hline
\end{tabular}
\caption{Transition samples per decision at planning depth $3$.}
\label{tab:ablation-rss-budget}
\end{table}

Despite this substantially larger sampling budget, RSS remains consistently below CAR-MCTS across the tested configurations. With $C=2$, RSS is often comparable to nominal UCT and in
several configurations performs worse. Thus, although RSS explicitly performs
robust planning, this does not translate into a consistent empirical robustness
advantage over the nominal planner under the tested model mismatch. Increasing
the branching factor to $C=3$ or $C=4$ generally improves RSS, but the gains
are inconsistent; even with tens of times more transition samples, its success
rate frequently remains close to that of nominal UCT.

\paragraph{Effect of transition-support width.}
To test whether the behavior of RSS is primarily caused by the broad
successor support of the $k=8$ kernel, we repeat the same experiment with
$k=1$, while keeping the remaining settings unchanged, shown in Figure \ref{fig:ablation-slipperiness-k1}.

\begin{figure}[h]
    \centering
    \includegraphics[width=0.7\linewidth]{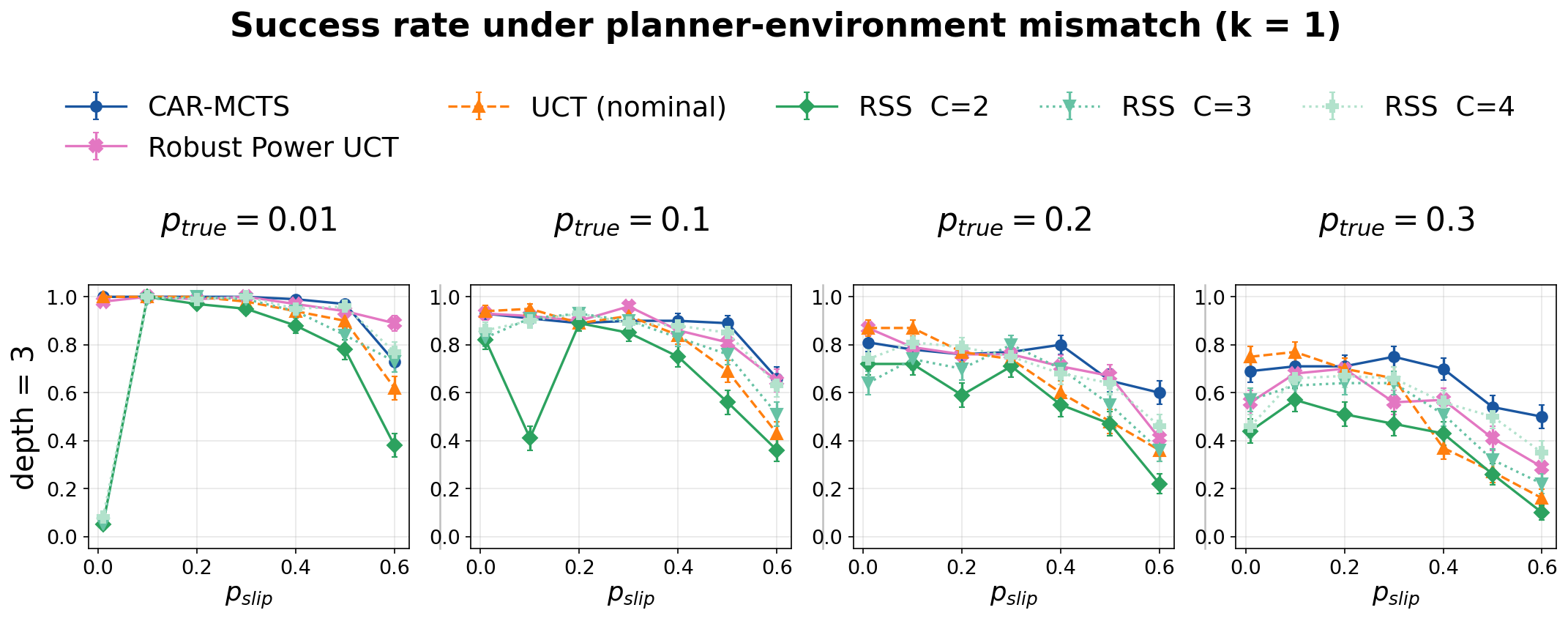}
    \caption{
    The same experiment as Figure~\ref{fig:ablation-slipperiness-k8}, with
    the ring radius reduced from $k=8$ to $k=1$.
    }
    \label{fig:ablation-slipperiness-k1}
\end{figure}

Reducing the ring radius substantially improves several competing methods,
showing that their performance is sensitive to the number of feasible next
states. In contrast, CAR-MCTS changes remarkably little between $k=8$ and
$k=1$. This is an important strength of the proposed approach: CAR-MCTS
maintains similar performance even when the transition support grows
dramatically, whereas the competing sampling-based planners benefit
substantially from the easier, narrower-support setting. RSS also improves in
several configurations, but $C=2$ remains frequently comparable to or worse
than nominal UCT, while the larger-$C$ variants still do not consistently
outperform it. Hence, the behavior observed for RSS at $k=8$ cannot be
explained solely by insufficient sampling of a broad successor distribution.

Overall, the ablation highlights three complementary strengths of CAR-MCTS.
First, its performance degrades more gracefully as the true environment becomes
increasingly slippery. Second, for a fixed true environment, it remains robust
as the planner's assumed transition model moves away from the true dynamics.
Third, its performance is far less sensitive to the width of the transition
support: increasing the number of feasible next states has only a limited
effect on CAR-MCTS, while substantially degrading the competing methods.
Together, these results show that CAR-MCTS remains effective under both model
uncertainty and broad transition dynamics, without requiring the large increase
in successor samples used by RSS.

% ==================================
% Ablation 2
% ==================================
\subsubsection{Ablation 2: Scalability with Planning Horizon}
\label{app:ablation-horizon}

We next isolate the effect of the planning horizon on online computation.
All environment and planner parameters are held fixed while varying
$H\in\{1,\ldots,10\}$. Experiments use a $100\times100$ Frozen Lake domain
with hole density $0.15$, ring radius $k=8$, and $\beta=1$. The planner uses
$p_{\mathrm{slip}}=0.3$, while the true execution dynamics use
$p_{\mathrm{true}}=0.1$, giving $\rho=0.197$. CAR-MCTS and Robust Power UCT
use $20$ MCTS simulations per decision, and results are averaged over
$20$ independent seeds.

\begin{figure}[h]
    \centering
    \includegraphics[width=0.6\linewidth]{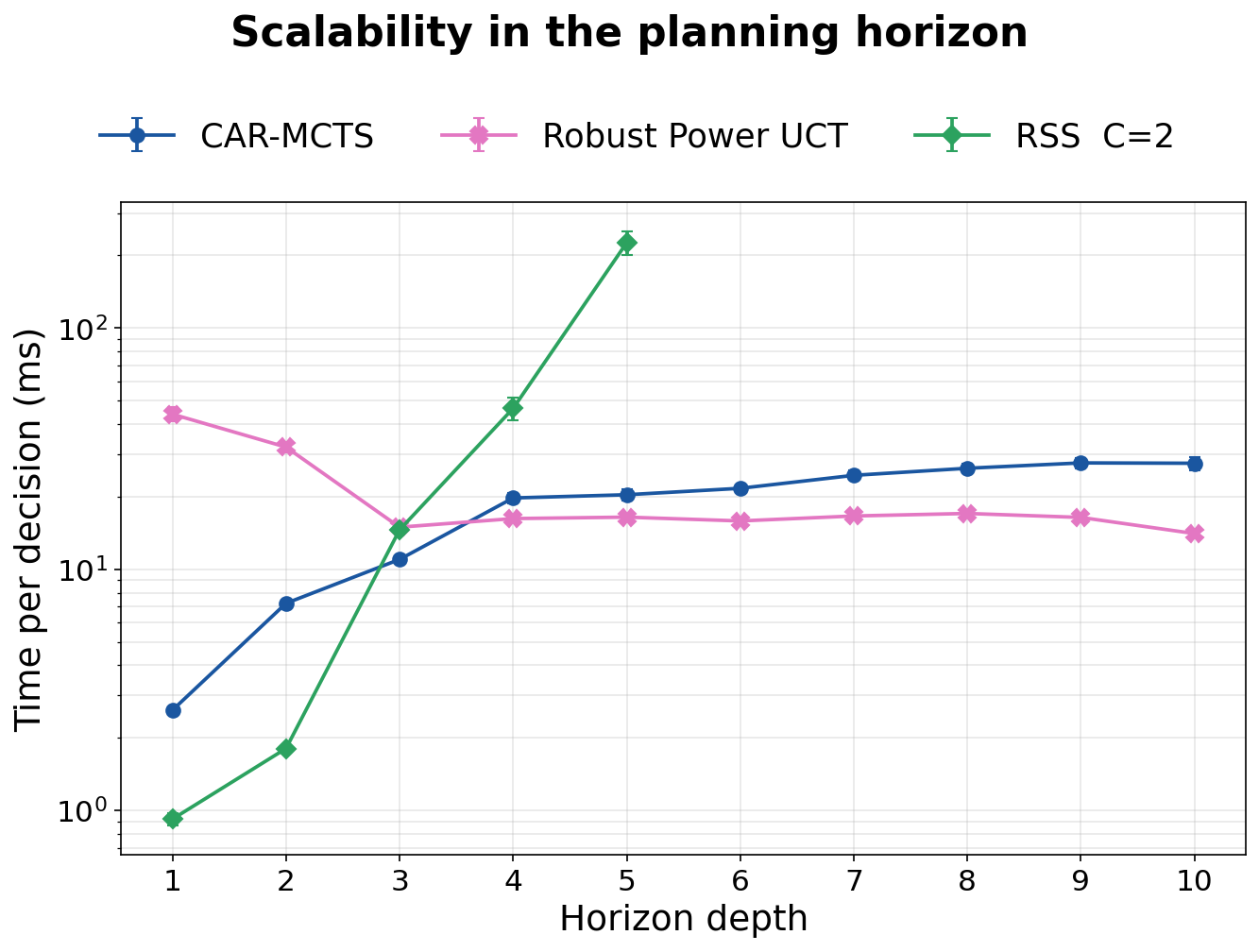}
    \caption{
    Decision time as a function of planning depth. CAR-MCTS and Robust Power
    UCT use $20$ simulations per decision, while RSS uses branching factor
    $C=2$. A $20$-second computation cap is imposed per decision; RSS reaches
    this limit at depth $6$, and deeper configurations are therefore not evaluated.
    }
    \label{fig:ablation-horizon}
\end{figure}

Figure~\ref{fig:ablation-horizon} shows a clear difference in scalability
with planning depth. CAR-MCTS remains within the tens-of-milliseconds regime
throughout the entire sweep, with decision time increasing only gradually as
the depth grows. Robust Power UCT similarly remains approximately constant
over the tested range. In contrast, RSS is initially the fastest method at
small depths, but its computational cost grows rapidly with each additional
planning level, increasing from below $1$ ms at depth $1$ to more than
$200$ ms at depth $5$. We impose a $20$-second computation cap per decision;
at depth $6$, RSS reaches this limit before completing the planning step.
Consequently, deeper RSS configurations are not evaluated.

The difference stems from the underlying search structure. RSS recursively
samples $C$ successor states for every action at each level of the tree, so
its computational cost grows rapidly with depth. In contrast, CAR-MCTS uses
a fixed simulation budget: increasing the planning depth lengthens the sampled
trajectories and their backups, but does not require recursively expanding all
sampled successor combinations. Consequently, CAR-MCTS supports substantially
deeper lookahead while retaining practical online decision times.

% ==================================
% Ablation 3
% ==================================
\subsubsection{Ablation 3: True Optimal-Action Selection}
\label{app:ablation-optimal-action}

Finally, we evaluate the quality of the action selected by each planner
directly, rather than through the stochastic outcome of an episode. For each
planning problem, we compute the optimal action offline using the true
transition model and measure the fraction of decisions in which each planner
selects this action. Importantly, the true model is used only as an evaluation
oracle and is not available to any of the planners.

We use a $4\times4$ Frozen Lake domain with ring radius $k=1$, search depth
$3$, and $20$ simulations for the MCTS-based methods. RSS uses
$C=2$, and each configuration is evaluated over $100$ independent seeds.
As in Ablation $1$ ~\ref{app:ablation-slipperiness}, the columns vary the true environment slip probability
$p_{\mathrm{true}}\in\{0.01,0.1,0.2,0.3\}$, while the $x$-axis varies the
planner-assumed $p_{\mathrm{slip}}$.

\begin{figure}[h]
    \centering
    \includegraphics[width=\linewidth]{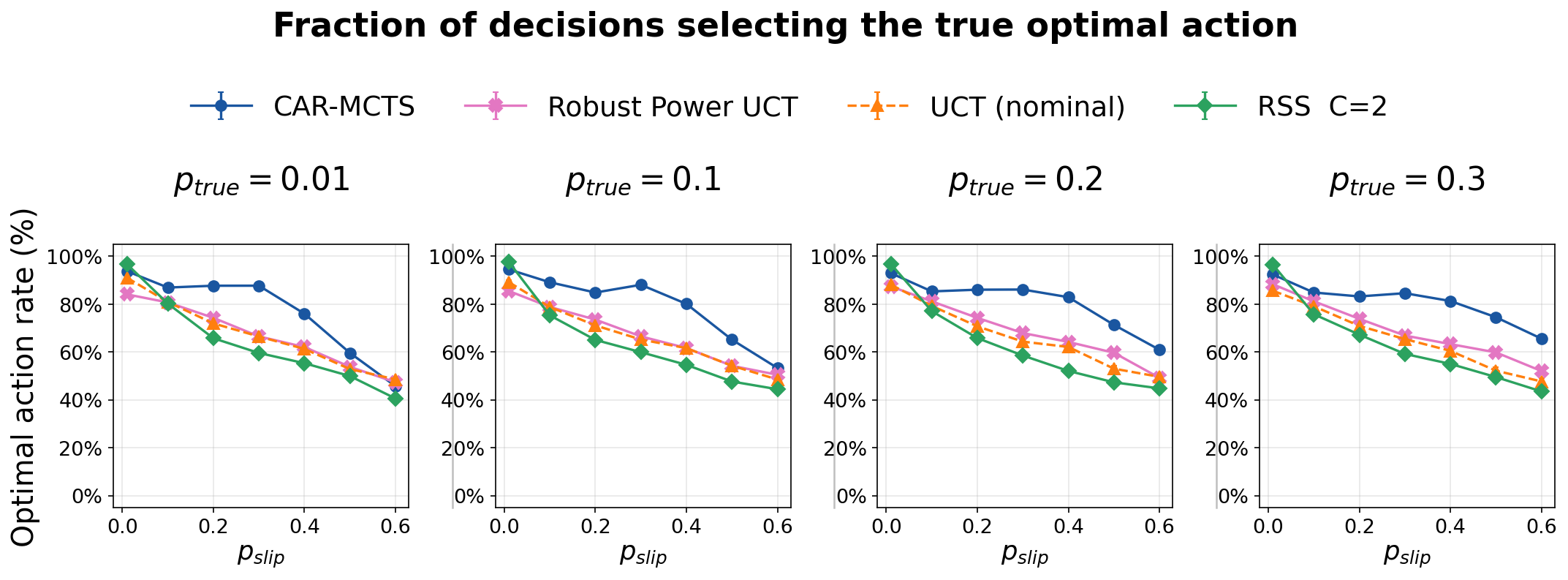}
    \caption{
    Fraction of decisions selecting the action that is optimal under the true
    transition model. Columns correspond to different values of
    $p_{\mathrm{true}}$, while the $x$-axis varies the planner-assumed
    $p_{\mathrm{slip}}$. The true model is used only for evaluation.
    }
    \label{fig:ablation-optimal-action}
\end{figure}

Figure~\ref{fig:ablation-optimal-action} shows that CAR-MCTS selects the true
optimal action more reliably than the competing planners over nearly the
entire sweep. When the true environment is only mildly slippery, all methods
perform well when the planning model is accurate, but their behavior diverges
as the assumed model moves away from the true dynamics. CAR-MCTS degrades more
gradually and maintains a substantially higher optimal-action selection rate
under increasing model mismatch.

The difference becomes particularly pronounced as the true environment itself
becomes more slippery. For $p_{\mathrm{true}}=0.2$ and $0.3$, CAR-MCTS
maintains a clear advantage over a broad range of $p_{\mathrm{slip}}$ values,
showing that it continues to identify the correct decision even when both the
environment is more difficult and the planner model is inaccurate. Robust
Power UCT generally improves over nominal UCT in parts of the sweep, but does
not approach the consistency of CAR-MCTS. RSS with $C=2$ is typically
comparable to or below nominal UCT, indicating that its robust formulation
does not translate into a consistent advantage in identifying the true
optimal action under the tested planning budget.

This experiment complements the episode-level results of Ablation~1 \ref{app:ablation-slipperiness}. Higher success rates can in principle result partly from favorable stochastic
transitions after an action is selected; here, the evaluation isolates the
planning decision itself. The results therefore show that CAR-MCTS's advantage
is not limited to better realized outcomes: it more frequently selects the
action that is optimal under the true environment dynamics. Together with the
previous ablations, this demonstrates stronger decision quality under model
mismatch, reduced sensitivity to broad transition support, and practical
scalability with planning depth.

% ==================================
% Ablation 4
% ==================================
\subsubsection{Ablation 4: Certification and Pruning in Practice}
\label{app:ablation-certification}

Finally, we examine how frequently the deterministic bounds used by CAR-MCTS
lead to root-action certification and branch pruning in practice. We evaluate
three grid sizes, $4\times4$, $8\times8$, and $16\times16$, while varying
$p_{\mathrm{slip}}$ and the planning budget over
$\{100,200,500,1000\}$ simulations. We report the fraction of decisions for
which the root action is certified and the fraction of branches pruned during
search.

\begin{figure}[h]
    \centering
    \includegraphics[width=0.8\linewidth]{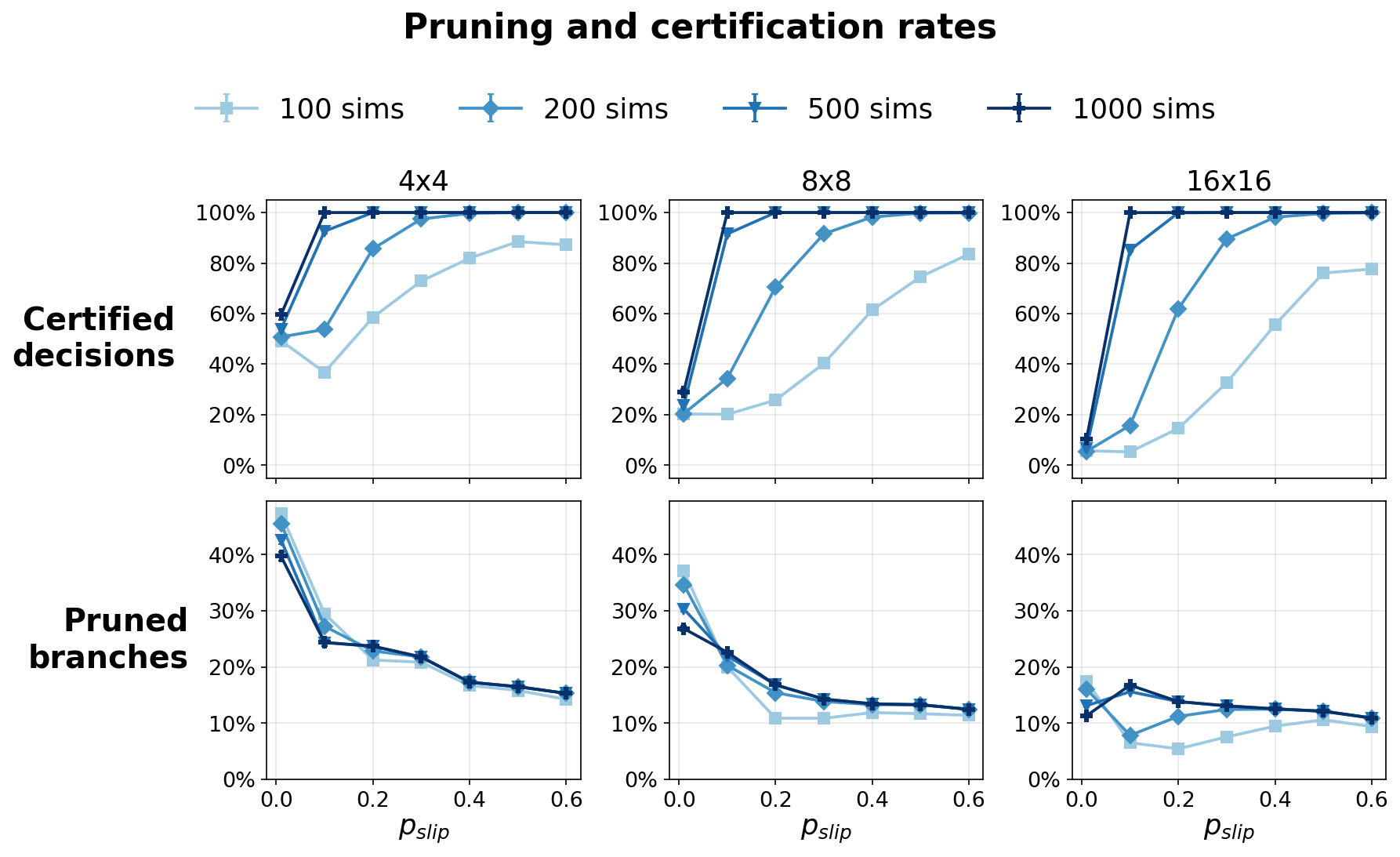}
    \caption{
    Root-action certification rate (top) and branch-pruning rate (bottom)
    for different grid sizes, planning budgets, and values of
    $p_{\mathrm{slip}}$.
    }
    \label{fig:ablation-certification}
\end{figure}

Figure~\ref{fig:ablation-certification} shows that root certification occurs
frequently in practice and increases strongly with the available planning
budget. Even with relatively small budgets of $100$--$200$ simulations, CAR-MCTS
certifies a substantial fraction of root decisions, with certification rates
often around $50\%$ or higher depending on the configuration. As the budget
increases, the deterministic intervals tighten further and the certification
rate rises sharply. At $500$ and $1000$ simulations, CAR-MCTS certifies the root decision in
nearly all runs over a large portion of the tested configurations.
This demonstrates that the certificates are not only
theoretical guarantees, but can become sufficiently tight during online
planning to provide an effective early-stopping criterion.

The lower row shows that the same bounds also eliminate a meaningful fraction of suboptimal branches during search. Interestingly, the measured pruning rate tends to decrease when certification becomes frequent. This is consistent with the two mechanisms playing complementary roles: while the root decision remains unresolved, continued sampling contracts the action-value bounds and enables dominated branches to be pruned; once the root action is certified, planning terminates and no further pruning is needed.

Overall, the ablation shows that CAR-MCTS uses its deterministic certificates
in two complementary ways to reduce online computation: locally, by pruning
branches that can no longer contain the optimal action, and globally, by
terminating planning once the root action is certified. The high certification
rates achieved with larger simulation budgets show that both mechanisms are
actively useful in practice rather than serving only as theoretical
guarantees.